\documentclass{article}

\usepackage{microtype}
\usepackage{graphicx}
\usepackage{subfigure}
\usepackage{booktabs} % for professional tables

\usepackage{hyperref}

\usepackage[accepted]{icml2024}

\usepackage{amsmath}
\usepackage{amssymb}
\usepackage{mathtools}
\usepackage{amsthm}

\usepackage[capitalize,noabbrev]{cleveref}

\theoremstyle{plain}
\newtheorem{theorem}{Theorem}[section]
\newtheorem{proposition}[theorem]{Proposition}

\theoremstyle{definition}

\theoremstyle{remark}

\usepackage[textsize=tiny]{todonotes}

\icmltitlerunning{Faster Repeated Evasion Attacks in Tree Ensembles}

\usepackage{amssymb}
\usepackage{enumitem} % to enumerate questions Q1, Q2, ....
\usepackage{bm}
\usepackage{siunitx}

\usepackage{tikz}
\usetikzlibrary{decorations.pathreplacing,calligraphy}
\definecolor{highlightcolor}{rgb}{0.858, 0.188, 0.478}
\tikzset{pics/.cd,
  subtree/.style={code={
      \path[draw] (0,0) -- (-0.3,-0.6) -- (0.3,-0.6) -- cycle;
      \node[] at (0.0, -0.4) {#1};
  }}
}

\begin{document}

\twocolumn[
\icmltitle{Faster Repeated Evasion Attacks in Tree Ensembles}




\begin{icmlauthorlist}
\icmlauthor{Lorenzo Cascioli}{kul}%{equal,yyy}
\icmlauthor{Laurens Devos}{kul}%{equal,yyy,comp}
\icmlauthor{Ondřej Kuželka}{cvut}%{comp}
\icmlauthor{Jesse Davis}{kul}%{comp}
\end{icmlauthorlist}

\icmlaffiliation{kul}{Department of Computer Science, KU Leuven, Leuven, Belgium}
\icmlaffiliation{cvut}{Department of Computer Science, Czech Technical University, Prague, Czech Republic}

\icmlcorrespondingauthor{Lorenzo Cascioli}{lorenzo.cascioli@kuleuven.be}

\icmlkeywords{Machine Learning, Tree Ensembles, Adversarial Examples}

\vskip 0.3in
]



\printAffiliationsAndNotice{}  % leave blank if no need to mention equal contribution

%\maketitle

\begin{abstract}
Tree ensembles are one of the most widely used model classes. 
However, these models are susceptible to adversarial examples, i.e., slightly perturbed examples that elicit a misprediction.
There has been significant research on designing approaches to construct such examples for tree ensembles. But this is a computationally challenging problem that often must be solved a large number of times (e.g., for all examples in a training set).  This is compounded by the fact that current approaches attempt to find such examples from scratch. In contrast, we exploit the fact that multiple similar problems are being solved. Specifically, our approach exploits the insight that  adversarial examples for tree ensembles tend to perturb a consistent but relatively small set of features. We show that we can quickly identify this set of features and use this knowledge to speedup constructing adversarial examples. 
\end{abstract}

\section{Introduction}

One of most popular and widely used class of models is tree ensembles which encompasses techniques such as gradient boosting \cite{friedman2001greedy} and random forests \cite{breiman2001random}. However, like other flexible model classes such as (deep) neural networks \cite{szegedy2013intriguing,goodfellow2014explaining}, they are susceptible to evasion attacks \cite{kantchelian2016evasion}.
That is, an adversary can craft an imperceptible perturbation that, when applied to an otherwise valid input example, elicits a misprediction by the ensemble.
There is significant interest in reasoning about tree ensembles to both generate such adversarial examples \cite{einziger19,zhang20} and perform empirical robustness checking \cite{kantchelian2016evasion,chen2019robustness,devos21a} where the goal is to determine how close the nearest adversarial example is. 

Generating adversarial examples is an NP-hard problem~\cite{kantchelian2016evasion}, which has spurred the development of approximate techniques~\cite{chen2019robustness,zhang20,devos21a}.  
These methods exploit the structure of the trees to find adversarial examples faster, for example, by using graph transformations \cite{chen2019robustness} or discrete (heuristic) search \cite{zhang20,devos21a,devos24multiclass-veritas}.
Still, these techniques can be slow, particularly if there is a large number of attributes in the domain. This is compounded by the fact that one often wants to generate large sets of adversarial examples.

A weakness to existing approaches is that they ignore the fact that adversarial example generation is often a sequential task where multiple similar problems are being solved in a row. That is, one has access to a large number of ``normal'' examples each of which should be perturbed to elicit a misprediction. Alas, existing approaches treat each considered example in isolation and solve the problem from scratch. However, there are likely regularities among the problems, meaning that the algorithms perform redundant work.
If these regularities can be identified efficiently and this information can be exploited to guide the search for an adversarial example, then the run time performance of repeated adversarial example generation can be improved. 

Studying these regularities in order to make adversarial example generation faster is an important problem. First, it advances our understanding of the nature of adversarial examples in tree ensembles and their generation methods. This might inspire improvements to generation methods, and in turn lead to better defense or detection methods. Second, model evaluation by verification \cite{ranzato2020,tornblom20,devos21a} is quickly becoming important as machine learning is applied in sensitive application areas. Being able to efficiently generate adversarial examples is crucial in the computation of empirical robustness (e.g.,~\cite{devos21a}), adversarial accuracy (e.g.,~\cite{vos21a-groot}), and for model hardening (e.g.,~\cite{kantchelian2016evasion}).

We propose a novel approach that analyzes previously solved adversarial example generation tasks to inform the search for subsequent tasks. 
Our approach is based on the observation that for a fixed learned tree ensemble, adversarial examples tend be generated by perturbing the same, relatively small set of features.
We propose a theoretically grounded manner to quickly find this set of features.
We propose two novel strategies to use the identified features to guide the search for adversarial examples, one of which is guaranteed to produce an adversarial example if it exists.
We apply our proposed approach to two different algorithms for generating adversarial examples~\cite{kantchelian2016evasion,devos21a}. Empirically, our approaches result in speedups of up to 35x and of 7.7x on average.

\section{Preliminaries}

We briefly explain tree ensembles, evasion attacks, and the two adversarial generation methods that we will use in the experiments. We assume a $d$-dimensional input space $\mathcal{X} \subseteq \mathbb{R}^d$ and binary output space $\mathcal{Y} = \{-1,1\}$. We focus on binary classification because most existing methods for verifying or generating adversarial examples for tree ensembles are designed for this setting~\cite{andriushchenko19,kantchelian2016evasion,devos21a}.

\subsection{Tree Ensembles}

Tree ensembles include popular algorithms such as (gradient) boosted decision trees (GBDTs) \cite{friedman2001greedy} (e.g., XGBoost \cite{chen2016xgboost}) and random forests \cite{breiman2001random} (e.g., as in scikit-learn \cite{scikit-learn}). A tree ensemble contains a number of trees and most implementations only learn binary trees.  A binary tree $T$ contains two types of nodes. \textit{Internal nodes} store references to a left and a right sub-tree, and a split condition on some attribute $f$ in the form of a less-than comparison $X_f < \tau $, where $\tau$ is the split value. \textit{Leaf nodes} have no children and only contain an output value. Each tree starts with a \textit{root node}, the only one without a parent. 
 
Given an example $x$, an individual tree is evaluated recursively starting from the \textit{root node}. In each internal node, the split condition is applied and if it is satisfied, then the example is sorted to the left subtree and if not it is sorted to the right one. This procedure terminates when a \textit{leaf node} is reached. The final prediction of the ensemble $\bm{T}(x)$ is obtained by combining the predicted leaf values for each tree in the ensemble.  In gradient boosting, the class probability is computed by applying a sigmoid transformation to the sum of the leaf values.

\subsection{Evasion Attacks}
\label{sec:adv_methods}

An \textit{evasion attack} involves carefully manipulating valid inputs $x$ into adversarial examples $\tilde{x}$ in order to evoke a misprediction \cite{kantchelian2016evasion}. More formally, we use the same definition of adversarial examples as used in existing work on tree ensembles~\cite{kantchelian2016evasion,hchen19robust,devos21a} and say that $\tilde{x}$ is an \textbf{adversarial example} for normal example $x$ when (1) ${\|\tilde{x} - x \|}_{\infty} < \delta$ where $\delta$ is a user-selected maximum distance (i.e., the two are sufficiently close), (2) the ensemble predicts the correct label for $x$, and (3) the model's predicted labels for $\tilde{x}$ and $x$ differ.

We now briefly describe the two existing adversarial example generation methods $\mathcal{A} : (\bm{T}, x, \delta, t_{\max}) \rightarrow \{ \mathit{SAT}(\tilde{x}), \mathit{UNSAT}, \mathit{TIMEOUT} \}$ used in this paper: \textit{kantchelian}~\cite{kantchelian2016evasion} and \textit{veritas}~\cite{devos21a}.
These methods take as input an ensemble $\bm{T}$, a normal example $x$, a maximum perturbation size $\delta$, and a timeout $t_{\max}$. They output $\mathrm{SAT}(\tilde{x})$, where $\tilde{x}$ is an adversarial example for $x$, $\mathrm{UNSAT}$, indicating that no adversarial example exists, or $\mathrm{TIMEOUT}$, indicating that no result could be found within timeout $t_{\max}$.
Timeouts are explicitly handled because adversarial example generation is NP-hard~\cite{kantchelian2016evasion}.

\textit{kantchelian} formulates the adversarial example generation task as a mixed-integer linear program (MILP) and uses a generic MILP solver (e.g., Gurobi~\cite{gurobi}). Specifically, \textit{kantchelian} directly minimizes the $\delta = {\| x - \tilde{x} \|}_\infty$ value. Given an example $x$, it computes:
\begin{equation}
        \min_{\tilde{x}} {{\| x - \tilde{x} \|}}_{\infty}  \quad\text{subject to}\quad \bm{T}(x) \neq \bm{T}(\tilde{x}).
        \label{eq:kan}
\end{equation}
This approach exploits the fact that a tree ensemble can be viewed as a set of linear (in)equalities. Three sets of MILP variables are used. \textit{Predicate variables} $p_i$ represent the split conditions, i.e., each $p_i$ logically corresponds to a split on an attribute $f$: $p_i \equiv f < \tau$. \textit{Leaf variables} $l_i$ indicate whether a leaf node is active. The \textit{bound variable} $b$ represents the $l_\infty$ distance between the original example $x$ and the adversarial example $\tilde{x}$.
Constraints between the variables encode the structure of the tree. A set of predicate consistency constraints encode the ordering between splits.
For example, if two split values $\tau_1 < \tau_2$ appear in the tree for attribute $f$, and  $p_1 \equiv f < \tau_1$ and $p_2 \equiv f < \tau_2$, then $p_1 \implies p_2$.
Leaf consistency constraints enforce that a leaf is only active when the splits on the root-to-leaf path to that leaf are satisfied.
Lastly, the mislabel constraint requires the output to be a certain class: for leaf values $v_i$, $\sum_i v_i l_i \lessgtr 0$.
The objective directly minimizes the \textit{bound variable}.

\textit{veritas} improves upon \textit{kantchelian} in terms of run time by formulating the adversarial example generation problem as a heuristic search problem in a graph representation of the ensemble (originally proposed by \cite{chen2019robustness}).
The nodes in this graph correspond to the leaves in the trees of the ensemble. Guided by a heuristic, the search then repeatedly selects compatible leaves. Leaves of two different trees are compatible when the conjunction of the split conditions along the root-to-leaf paths of the leaves are logically consistent. 
For a given $\delta$, \textit{veritas} solves the following optimization problem:\footnote{Note that we are abusing terminology: here, $\bm{T}(x)$ is the predicted probability. Previously, it was the predicted label.}
\begin{equation}
        \mathop{\mathrm{optimize}}_{\tilde{x}}\ \bm{T}(\tilde{x}) \quad\text{subject to}\quad
        {\| x - \tilde{x} \|}_\infty < \delta
        \label{eq:veritas}
\end{equation}
The output of the model $\bm{T}(\tilde{x})$ is maximized when the target class for $\tilde{x}$ is positive, and minimized otherwise.
While \textit{veritas} can also be used to directly optimize $\delta$, in this paper we will use a predefined $\delta$ for \textit{veritas}.

\section{Method}
\label{sec::method}

Adversarial example generation methods like \textit{kantchelian} and \textit{veritas} are typically applied in the following setting:

\begin{tabular}{rp{6cm}}
     \textbf{Given} & a tree ensemble $\bm{T}$, a set of test examples $V$, and a maximum perturbation size $\delta$, \\
     \textbf{Generate} &  adversarial examples for each $x \in V$.
\end{tabular}

The goal of this paper is to exploit the fact that adversarial examples are sequentially generated for each example in $V$. By analyzing previously found adversarial examples, we aim to improve the efficiency of adversarial example generation algorithms by biasing the search towards the perturbations that most likely to lead to an adversarial example. 

Our hypothesis is that some parts of the ensemble are disproportionately sensitive to small perturbations, i.e., crossing the thresholds of split conditions in these parts of the ensemble results in large changes in the predicted value. Prior work has hypothesized that robustness is related to fragile features and that such features are included in models because learners search for any signal that improves predictive performance~\cite{NEURIPS2019_e2c420d9}. One would expect that the attributes used in the split conditions in these disproportionately sensitive parts are exploited by adversarial examples more frequently than other attributes.  

Figure~\ref{fig:modfreq} illustrates this point by showing how often each attribute is perturbed in a set of a 10\,000 adversarial examples generated by \textit{kantchelian} for two different datasets.
The bar plots distinguish three categories of attributes: attributes that are never modified by any adversarial example (left), attributes that are modified by at least one but at most 5\% of all adversarial examples (middle), and attributes that are modified by more than 5\% of the adversarial examples.
Less than 10\% of the attributes are used by more than 5\% of the adversarial examples.
Thus the two questions are \textbf{(1) how can one identify these frequently-modified attributes and (2) how can algorithms exploit this knowledge to more quickly generate adversarial examples. }

\begin{figure}[h]
    \centering
    \footnotesize
    \includegraphics{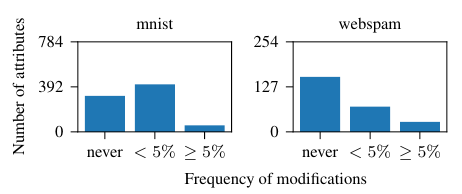}
    \caption{Bar plots showing that most attributes are not modified by the majority of adversarial examples (\textit{mnist} and \textit{webspam} only). The leftmost bar shows the number of attributes that are never changed by any of the 10,000 generated adversarial examples. The middle bar shows the number of attributes that are modified at least once but at most by 5\% of the adversarial examples. The rightmost bar shows the number of frequently modified features.}
    \label{fig:modfreq}
\end{figure}

At a high level, our proposed approach has two parts.
The first part simplifies the search for adversarial examples by only allowing perturbations to a limited subset of features. This is accomplished by exploiting the knowledge that certain feature values are fixed, which enables simplifying the ensemble by pruning away branches that can never be reached.
The second part identifies a subset of commonly perturbed features by counting how often each feature is perturbed by adversarial examples (Section~\ref{sec:identify_attr}). The size of this subset is determined by applying a theoretically grounded statistical test.

\subsection{Modifying the Search Procedure}%
\label{sec:modify_search}

Our proposed approach speeds up the adversarial example generation procedure by limiting the scope of the adversarial perturbations to a subset of features $F_{S}$.
This section assumes that we are given such a subset of features. The next section covers how to identify these features.

We consider three settings: \textit{full}, \textit{pruned}, and \textit{mixed}.
The \textit{full} setting corresponds to the original configuration of \textit{kantchelian} and \textit{veritas}: the methods may perturb any attribute within a certain maximum distance $\delta$. That is, for each attribute $f \in F$ with value $x_f$, the attribute values are limited to $[x_f -\delta, x_f + \delta]$.
Algorithm~\ref{alg:repgen} summarizes the \textit{pruned} and \textit{mixed} approaches. We now describe both in greater detail.% settings.

\paragraph{Pruned Approach} The \textit{pruned} setting disallows modifications to the attributes in the \textit{non-selected} set of attributes $F_{\mathit{NS}} = F \setminus F_{S}$. We accomplish this by pruning the trees in the ensemble.
Any node splitting on attributes in $F_{\mathit{NS}}$ is removed. Its parent node is directly connected to the only child node that can be reached by examples with the fixed value for the attribute. Figure~\ref{fig:pruning} shows an example of this procedure. We refer to this procedure as $\textsc{Prune}(\bm{T}, F_S, x)$.
The adversarial example methods can be applied as normal to the pruned ensemble, but they will only generate adversarial examples with perturbations to the attributes in $F_{S}$. Pruning simplifies the MILP problem of \textit{kantchelian} because all predicate variables $p_i$ that correspond to splits in internal nodes of pruned subtrees, and leaf variables $l_i$ that correspond to leaves of pruned subtrees can be removed from the mathematical formulation.
For \textit{veritas}, the search space is reduced in size because the pruned leaves are  removed from the graph representation of the ensemble.
Hence, for both systems, on average, the problem difficulty is reduced by pruning the ensembles.

Pruning the trees does not affect the validity of generated adversarial examples: If $\tilde{x}$ is an adversarial example generated for a normal example $x$ generated on a pruned ensemble, then $\tilde{x}$ is also an adversarial example for the full ensemble. %$\bm{T}_{\mathit{full}}$.}
\begin{proposition} Given normal example $x$ that is correctly classified by the full ensemble $\bm{T}_{\mathit{full}}$. Let $\bm{T}_{\mathit{prun}} = \textsc{Prune}(\bm{T}_{\mathit{full}}, F_{S}, x)$ and $\tilde{x}= \mathcal{A}(\bm{T}_{\mathit{prun}},x,\delta, t_{\max})$ (i.e., $\bm{T}_{\mathit{prun}}(x) \neq \bm{T}_{\mathit{prun}}(\tilde{x})$ and ${\| x - \tilde{x} \|}_\infty < \delta$). Then it holds that $\bm{T}_{\mathit{full}}(x) \neq \bm{T}_{\mathit{full}}(\tilde{x})$.

\label{prop:prune}
\end{proposition}
\begin{proof}
    Because only branches not visited by $x$ are removed, $\bm{T}_{\mathit{prun}}(x) = \bm{T}_{\mathit{full}}(x)$.
    The values for features in $F_{\mathit{NS}}$ are fixed, so these values are equal between $x$ and $\tilde{x}$.
    Hence, $\tilde{x}$ only visits branches in $\bm{T}_{\mathit{full}}$ that are also in $\bm{T}_{\mathit{prun}}$.
    Therefore, $\bm{T}_{\mathit{prun}}(\tilde{x}) = \bm{T}_{\mathit{full}}(\tilde{x})$
\end{proof}
However, an \textit{UNSAT} generated on a pruned ensemble is inconclusive. That is, it might still be the case that an adversarial example exists for the full ensemble, albeit one with perturbations to features in $F_{\mathit{NS}}$.
The \textit{pruned} setting generates a \textbf{false negative} if it reports \textit{UNSAT}, yet the \textit{full} setting reports \textit{SAT}.

\paragraph{Mixed Approach}
The \textit{mixed} setting takes advantage of the fast adversarial generation capabilities of the \textit{pruned} setting, but falls back to the \textit{full} setting when the \textit{pruned} setting returns an \textit{UNSAT} or times out.
A much stricter timeout $t_{\max}^{\mathit{prun}}$ is used for the \textit{pruned} setting to fully take advantage of the fast \textit{SAT}s, while avoiding spending time on an uninformative \textit{UNSAT}. %In case of a timeout, we re-run the \textit{full} problem too. 
The \textit{mixed} setting is guaranteed to find an adversarial example if the \textit{full} setting can find one.
\begin{theorem}
Assume a normal example $x$ and maximum distance $\delta$. If an adversarial example can be found for the full ensemble $\bm{T}_{\mathit{full}}$, then the \textit{mixed} setting is guaranteed to find an $\tilde{x}$ such that ${\|x-\tilde{x}\|}_\infty < \delta$ and $\bm{T}_{\textrm{full}}(x) \neq \bm{T}_{\textrm{full}}(\tilde{x})$.
\label{thrm:soundness}
\end{theorem}
\begin{proof}
    The mixed setting first operates on the pruned ensemble $\bm{T}_{\textrm{prun}}$ using a tight timeout and optimizes Equation~\ref{eq:kan} or~\ref{eq:veritas} using \textit{kantchelian} or \textit{veritas} respectively.
    This returns (1) an adversarial example $\tilde{x}$, (2) an \textit{UNSAT} or (3) times out.
    In case (1), the generated adversarial example $\tilde{x}$ is also an adversarial example for the full ensemble (Prop~\ref{prop:prune}).
    In cases (2) and (3), the \textit{mixed} setting falls back to the \textit{full} setting operating on the full ensemble $\bm{T}_{\mathit{full}}$ with the same timeout.
    Hence, it inherits the full method's guarantees. % By falling back on the full method, the guarantees are inherited.
\end{proof}

\begin{algorithm}[h]
   \caption{Fast repeated adversarial example generation}\label{alg:repgen}
\begin{algorithmic}[1]
    \STATE {\bfseries parameters:}
        maximum perturbation size $\delta$,
        timeouts $t^{\mathit{full}}_{\max}$ and $t^{\mathit{prun}}_{\max}$ for \textit{full} and \textit{pruned},
        generation method $\mathcal{A} : (\bm{T}, x, \delta, t) \rightarrow \{ \mathrm{SAT}(\tilde{x}), \mathrm{UNSAT}, \mathrm{TIMEOUT} \} $ 
    \vspace{0.5em}
    \FUNCTION{$\textsc{Generate}(\bm{T}_{\mathit{full}}, \mathcal{D}, F_{S}, \textit{mixed}\ \text{flag})$}
    \STATE $\tilde{\mathcal{D}} \leftarrow \emptyset$
    \FOR{$x \in \mathcal{D}$}
        \STATE $\bm{T}_{\mathit{prun}} \leftarrow \textsc{Prune}\left(\bm{T}_{\mathit{full}}, F_{S}, x\right)$
        $\qquad$ (Sec.~\ref{sec:modify_search})
        \STATE $\alpha \leftarrow \mathcal{A}\left(\bm{T}_{\mathit{prun}}, x, \delta, t_{\max}^{\mathit{prun}}\right)$
        \IF{$ \alpha \neq \mathrm{SAT}(\tilde{x})\ \land$ \textit{mixed} flag set}
            \STATE $\alpha \leftarrow \mathcal{A}\left(\bm{T}_{\mathit{full}}, x, \delta, t_{\max}^{\mathit{full}}\right)$
        \ENDIF
        \STATE $\tilde{\mathcal{D}} \leftarrow \tilde{\mathcal{D}} \cup \{ \alpha \}$
    \ENDFOR
    \STATE {\bfseries return:} $\tilde{\mathcal{D}}$
    \ENDFUNCTION
\end{algorithmic}
\end{algorithm}

\begin{figure}
    \footnotesize
    \centering
    \begin{tikzpicture}
        [level distance=6mm,
           every node/.style={inner sep=3pt},
        ]
        \node [rectangle,draw] {$\textsc{Height}<200$}
        child {node (l) {} edge from parent}
        child {node [rectangle,draw] {$\textsc{Age}<50$} edge from parent
            child {node (rl) {}}
            child {node (rr) {}}};

        \pic at (l) {subtree=a};
        \pic at (rl) {subtree=b};
        \pic at (rr) {subtree=c};

        % arrow from original to combined
        \draw[->,thick] (1.9,0) -- (2.5,0);

        \node [rectangle,draw] at (4, 0) {$\textsc{Height}<200$}
        child {node (pl) {} edge from parent}
        child {node (pr) {} edge from parent};

        \pic at (pl) {subtree=a};
        \pic at (pr) {subtree=c};

        %\node [rectangle,draw] at (3.4,0) {$\textsc{Age}<40$}
        %child {node {4} edge from parent}
        %child {node {12} edge from parent};
    \end{tikzpicture} 
    \caption{An example tree using two attributes $\textsc{Height}$ and $\textsc{Age}$ (left). Suppose $F_NS=\{\textsc{Age}\}$. Given an example where $\textsc{Age}=55$, we can prune away the internal node splitting on $\textsc{Age}$. %keeping only the one reachable subtree (right). 
    In the resulting tree (right), subtree (b) is pruned because it is unreachable given that $\textsc{Age}=55$ and only subtrees (a) and (c) remain.}
    \label{fig:pruning}
\end{figure}
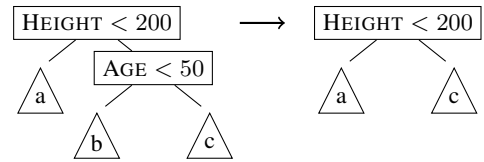

\subsection{Identifying Relevant Features}%
\label{sec:identify_attr}

A good subset of relevant attributes $F_S$ should satisfy two properties.
First,  it should minimize the number of false negatives, which occur when the \textit{pruned} approach reports \textit{UNSAT}, but the \textit{full} approach reports \textit{SAT}.
Second, the feature subset should be small. The smaller $F_{S}$ is, the more the ensemble can be pruned, and the faster the speedup is. These two objectives are somewhat in tension. Including more features will reduce the number of false negatives, but limit speeds up that are possible whereas using a very small subset will restrict the search too much resulting in many false negatives (or slow calls to the full search in the \textit{mixed} setting).
The procedure is given in Algorithm~\ref{alg:subset}.

We address the first requirement by adding features to the subset that  are frequently perturbed by adversarial examples. We rank features by counting how often each one differs between the perturbed adversarial examples in $\tilde{\mathcal{D}}$ so far and their corresponding normal examples in $\mathcal{D}$.

The second requirement is met by statistically testing whether the identified subset guarantees that the false negative rate is smaller than a given threshold with probability at least $1-\delta$, for a specified confidence parameter $\delta$.
If it is not guaranteed, then the subset is expanded. This is done at most 4 times for subsets of 5\%, 10\%, 20\%, 30\% of the features. If all tests fail, then a final feature subset of 40\% of the most commonly modified features is used. We do not go beyond 40\% because using the full feature set is then more efficient ($\textsc{ExpandFeatureSet}(F_S, \mathcal{D}, \tilde{\mathcal{D}})$ in Algorithm~\ref{alg:subset}). Each test is executed on a small set of $n$ generated adversarial examples. A first zeroth set is used merely for obtaining the first feature counts.

Next we give the details of how these tests are performed.
Take $\mathcal{D}_F = (x_1,x_2,\dots,x_N)$ the dataset we use to find the feature subset $F_S$. We define $\mathbf{v} = (v_1,v_2,\dots,v_N)$ to be the binary vector such that $v_i = 1$ if the \textit{pruned} search with the feature subset $F_S$ returns \textit{UNSAT} for the example $x_i$ but the \textit{full} search returns \textit{SAT}, and $v_i = 0$ otherwise. Then the true false negative rate corresponding to $F_S$ can be written as $\textit{FNR} = \frac{1}{N} \sum_{i=1}^N v_i$. Now, the small set of $n$ examples from which we are estimating the false negative rate is a random vector $\mathbf{X} = (X_1,X_2,\dots,X_n)$ sampled without replacement from $\mathcal{D}_F$. We also define $\mathbf{V} = (V_1,V_2,\dots,V_n)$ where $V_i$ is the random variable defined analogically to how we defined $v_i$. It follows that $\sum_{i=1}^n V_i$ is distributed as a hypergeometric random variable. Our null hypothesis is that $\textit{FNR}$ is greater than the threshold $\tau$. We reject the hypothesis if $\bar{V} = \frac{1}{n} \sum_{i=1}^n V_i$ takes a value smaller than the threshold by more than a margin $\Delta$. Next we bound the probability that, this happens, i.e., the probability $P\left[\bar{V} \leq \tau - \Delta\right]$, under the null hypothesis $\textit{FNR} \geq \tau$:
\begin{align}
    P[\bar{V} \leq \tau - \Delta] &= P[\bar{V} - \textit{FNR} \leq \tau - \Delta - \textit{FNR}] \nonumber \\
    &\leq P[ \bar{V} - \textit{FNR} \leq -\Delta],\label{eq:stattest_notau}
\end{align}
where the second inequality follows from the null hypothesis. 
Due to how we defined $\bar{V}$, we have $\mathbb{E}[\bar{V}] = \textit{FNR}$. Moreover, $n \cdot \bar{V}$ is distributed as a hypergeometric random variable, therefore we can use an exponential bound from \citet{greene2017exponential}, stated below in Theorem \ref{thrm:greene}, to bound the probability. Since the existing theorem bounds $P[\bar{X}_n - \mu \geq \varepsilon]$ instead of our $P[\bar{X}_n - \mu \leq -\varepsilon]$, we provide the needed manipulations after the theorem. %agreed with JD's edit

\begin{theorem}[\citet{greene2017exponential}]
    \label{thrm:greene}
    Let $\bar{v} \sim \mathit{Hypergeometric(n, D, N)}$, a margin $\Delta$ as in Equation~\ref{eq:stattest_notau}, and $\lambda = \Delta \sqrt{n}$.
    Suppose $N > 4$ and $2 \leq n < D \leq N/2$.
    Then, for all $0 < \Delta < 1/2$:
    {\small
    \begin{align*}
        &P[ \sqrt{n}\left(\bar{v} - \mu \right) \geq \lambda ]
        \leq \sqrt{\frac{1}{2\pi \lambda ^2}} 
        \left( \frac{1}{2} \right)
        \\
        &\quad\cdot
        \sqrt{
            \left(
                \frac{N - n}{N}
            \right)
            \left(
                \frac{\sqrt{n} + 2\lambda}{\sqrt{n} - 2\lambda}
            \right)
            \left(
                \frac{N - n + 2\sqrt{n} \lambda }{N - n - 2\sqrt{n} \lambda }
            \right)
        }
        \\
        &\quad\cdot
        \exp\left(
            - \frac{2}{1-\frac{n}{N}} \lambda^2
        \right)
        \exp\left(
            - \frac{1}{3}
            \left(
                1+\frac{n^3}{{(N-n}^3}
            \right)
            \frac{\lambda^4}{n}
        \right).
    \end{align*}
    }
\end{theorem}

We define $W_i = 1-V_i$, $\mu_W = 1-\textit{FNR}$ and $\Delta = \lambda/\sqrt{n}$. It is clear that if $n \cdot \bar{V}$ is hypergeometric, so is $n \cdot \bar{W} = \sum_{i=1}^n W_i$.\footnote{Moreover, for any reasonably small acceptable FNR threshold $\tau$, with high probability we will either get $\bar{V} > \tau$ or it will hold $n < D$ where $D$ is the parameter of the hypergeometric distribution for $n \cdot \bar{W}$, allowing us to use the bound from Theorem \ref{thrm:greene}.} Then we can write:
{\footnotesize
\begin{align*}
    P[\sqrt{n}(\bar{W} - \mu_w) \geq \lambda]
    &= P\left[\bar{W} - \mu_w \geq \frac{\lambda}{\sqrt{n}} \right] \nonumber\\
    &= P\left[ \frac{1}{n}\sum_{i=1}^n (1-V_i) - 1 + \textit{FNR} \geq \frac{\lambda}{\sqrt{n}} \right] \nonumber\\
    &= P\left[ -\frac{1}{n}\sum_{i=1}^n V_i + \textit{FNR} \geq \frac{\lambda}{\sqrt{n}} \right] \nonumber\\
    &= P\left[ \bar{V} - \textit{FNR} \leq -\Delta \right].
\end{align*}}

Here, the last expression is what we need to bound and the first is what Theorem \ref{thrm:greene} bounds.

In the algorithm, we are given a confidence parameter $\eta$ and we determine $\Delta$ using the bound in Theorem~\ref{thrm:greene} so that the probability of incorrectly selecting a too small subset $F_S$ is smaller than $\eta$. 
Because we execute the test 4 times, we apply a union-bound correction of factor 4.
Choosing a confidence of 90\%, we extract $\Delta \in \{1/n, 2/n, \ldots, 1/2\}$ by computing the bound and stopping at the smallest $\Delta$ such that 
$P\left[ \bar{V} \leq \tau - \Delta \right] < \eta/4$. 
Note that there is a trade-off. The higher $n$, the better the statistical estimates and the counts are, but also the more examples we process with a potentially suboptimal feature subset.

\begin{algorithm}[h]
   \caption{Find feature subset}\label{alg:subset}
\begin{algorithmic}[1]
    \STATE $F_S \leftarrow \emptyset$
    \FOR{$k \in 0..4$}
        \STATE $\tilde{\mathcal{D}} \leftarrow \textsc{Generate}(\bm{T}, \mathcal{D}[kn, k(n+1)], F_{S}, \mathit{true}) \quad $%(Alg.~\ref{alg:repgen})\label{algline:v}
        \STATE $\bar{v} \leftarrow \frac{1}{n} \times $ number of false negatives in $\tilde{\mathcal{D}}$
        \STATE {\bfseries if} $\bar{v}$ exceeds $\tau-\Delta$, {\bfseries then} $F_S$, selecting most frequently perturbed features in $\tilde{\mathcal{D}}$ first.
        \STATE {\bfseries else break}  the loop
    \ENDFOR
\end{algorithmic}
\end{algorithm}

\section{Experiments}
\label{sec::experiments}
Empirically, we address the following questions:
\begin{enumerate}[label={Q\arabic*}]
\itemsep0em 
    \item Is our approach able to improve the run time performance of generating adversarial examples? 
    \item How does ensemble complexity affect our approach's performance?
    \item What is our empirical false negative rate?
    
\end{enumerate}

Because the described procedure is based on identifying a subset of relevant features, it makes sense to exploit it only when the dataset has a large number of dimensions.
Therefore, we present numerical experiments for ten classification tasks on high-dimensional datasets, using both tabular data and image data, as shown in Table \ref{tab::datasets}. 

\begin{table}[h]
\footnotesize
\begin{center}
%\caption{Datasets dimensionalities (\#F), adopted values of max allowed perturbation $\delta$, and XGBoost hyperparameter settings. Each ensemble $\bm{T}$ has maximum tree depth d and contains M trees. The learning rate is $\eta$. Links to datasets sources are in the supplement.}%
 \caption{Datasets’ characteristics: \textit{N} and \textit{\#F} are the number of examples and the number of features. We also report the adopted values of max allowed perturbation $\delta$, for XGBoost and Random forest. %{\color{highlightcolor} Because the latter is typically more robust, in some cases we used a larger $\delta$ in order to have an experimental setting similar to the one for XGBoost.} \lorenzo{Shall we say anything more about how we chose them and why some of those for RF are bigger (answer is otherwise we have UNSATs too frequently)?} 
 \textit{higgs} and \textit{prostate} are random subsets of the original, bigger datasets. Multi-class classification datasets were converted to binary classification: for \textit{covtype} we predict majority-vs-rest, for \textit{mnist} and \textit{fmnist}, we predict classes 0-4 vs. classes 5-9, and for \textit{sensorless} classes 0-5 vs. classes 6-10.}
\label{tab::datasets}
\vspace{1em}
\begin{tabular}{lrrrr}
    \toprule
    \textbf{Dataset} & N & \#F & $\delta$ XGB & $\delta$ RF  \\
    \midrule
    covtype & 581k & 54 & 0.1 & 0.3 \\
    fmnist & 70k & 784 & 0.3 & 0.3 \\
    higgs & 250k & 33 & 0.08 & 0.08\\
    miniboone & 130k & 51 & 0.08 & 0.08 \\
    mnist &  70k & 784 & 0.3 & 0.3 \\
    prostate & 100k & 103 & 0.1 & 0.2 \\
    roadsafety & 111k & 33 & 0.06 & 0.12 \\
    sensorless & 58.5k & 48 & 0.06 & 0.12  \\
    vehicle & 98k & 101 & 0.15 & 0.15\\
    webspam & 350k & 254 & 0.04 & 0.06 \\
    \bottomrule
\end{tabular}
\end{center}
\end{table}

\subsection{Experimental Setup}
\label{sec::experiments_setup}
We apply 5-fold cross validation for each dataset. We use four of the folds to train an XGBoost or a random forest ensemble $\bm{T}$.
From the test set, we randomly sample 10\,000 normal examples and attempt to generate adversarial examples by perturbing each one using the \textit{veritas} or \textit{kantchelian} attack.  
Each ensemble's hyperparameters are tuned using the grid search described in Appendix~\ref{sec::datasets}. 
The experiments ran on an Intel(R) E3-1225 CPU with 32GiB of memory.

The \textit{pruned} and \textit{mixed} settings work as follows.
We use the procedure from Section~\ref{sec:identify_attr} to select  a subset of relevant features.
Using Theorem~\ref{thrm:greene}, we find $n=100$ and $\Delta=0.1$. This gives us a $1-\eta = 90\%$ confidence that our true false negative rate is below $\tau=25\%$.
We then apply Algorithm~\ref{alg:subset}: we generate 5 sets of $n$ adversarial examples to (1) find which features are perturbed most often and (2) determine the size of the feature subset $F_S$.
After Algorithm~\ref{alg:subset} terminates, $F_S$ is fixed, and we run the \textit{pruned} and \textit{mixed} settings on all the remaining test examples (Algorithm~\ref{alg:repgen}).

We set a timeout of one minute for the \textit{full} setting, and a much stricter timeout of 1 (\textit{kantchelian}) or 0.1 (\textit{veritas}) seconds in the \textit{pruned} setting. We can be stricter with \textit{veritas} as it is an approximate method that is faster than the exact \textit{kantchelian}. 

\subsection{Q1: Run time}
\label{sec::runtime}

Table~\ref{tab::res_4} reports the average run time for the \textit{full} setting and the average speedup given by the \textit{pruned} and \textit{mixed} settings.  Our approach consistently speeds up \textit{kantchelian} with the \textit{pruned} approach yielding speedups between 2.4x-25.4x, and \textit{mixed} between 1.6x-7.1x. 
Using \textit{veritas}, we achieve speedups of between 1.6x-35.9x with the \textit{pruned} approach, and  between 1.0x-4.5x with the \textit{mixed} approach.  

Generally, \textit{kantchelian} benefits slightly more than \textit{veritas} regardless of which setting is used. This is because \textit{veritas} is already an approximate approach and its existing heuristics leave less room for improvement. In contrast, the feature pruning can greatly simplify the MILP problem, which consequently leads to faster run times.

The story is more complicated when considering the ensemble type. On XGB ensembles, both settings offer consistent wins. The \textit{mixed} setting falls back to the \textit{full} search on average 8\% of the time, regardless of the attack.\footnote{See Table~\ref{tab::res_extended} in the supplement} This helps it achieve a speedup by taking advantage of the fast \textit{SAT} results of the \textit{pruned} setting while still offering the theoretical guarantee from Theorem~\ref{thrm:soundness}.  

However, generating adversarial examples is more difficult for random forests (RF) than XGB.\footnote{When running \textit{kantchelian} the \textit{full} search hit the global time out of 6 hours, meaning that it terminated before attempting to generate 10,000 adversarial examples on six of the datasets.}  This leads to  the \textit{pruned} strategy offering larger wins than for XGB ensembles. %However, the \textit{mixed} strategy's benefits are sometimes  \lorenzo{there was a hole here...}.
For the RF ensembles, the \textit{pruned} setting often hits its timeout limit or fails more often, which leads to calls to the \textit{full} search 22\% of the time. 

% Table
%%% KANTCHELIAN - XGB
\begin{table*}
\footnotesize
\centering

\caption{Average run times and speedups to verify 10\,000 test examples using \textit{kantchelian}/\textit{veritas} on an XGBoost/random forest ensemble for all three approaches: \textit{full}, \textit{pruned} and \textit{mixed}. A star denotes that the dataset exceeded the six hour global timeout.}
\vspace{1em}
\label{tab::res_4}
%\begin{center}\bf{Kantchelian, XGBoost}\end{center}
%\begin{tabular}{@{}lS[table-format=2.1]S[table-format=2.1]S[table-format=2.1]S[table-format=2.1]S[table-format=2.1]S[table-format=2.1]S[table-format=2.1]S[table-format=2.1]S[table-format=2.1]S[table-format=2.1]S[table-format=2.1]S[table-format=2.1]@{}}
\begin{tabular}{@{}lS[table-format=2.1]rrS[table-format=2.1]rrS[table-format=2.1]rrS[table-format=2.1]rr@{}}
\toprule
 %& \multicolumn{6}{c}{\em kantchelian} & \multicolumn{6}{c}{\em veritas} \\
%\cmidrule(lr){2-7}
%\cmidrule(lr){8-13}

& \multicolumn{3}{c}{\textbf{Kantchelian XGB}}
& \multicolumn{3}{c}{\textbf{Kantchelian RF}}
& \multicolumn{3}{c}{\textbf{Veritas XGB}}
& \multicolumn{3}{c}{\textbf{Veritas RF}} \\
\cmidrule(lr){2-4}
\cmidrule(lr){5-7}
\cmidrule(lr){8-10}
\cmidrule(lr){11-13}
& {\textit{full}} & {\textit{pruned}} & {\textit{mixed}}
& {\textit{full}} & {\textit{pruned}} & {\textit{mixed}} 
& {\textit{full}} & {\textit{pruned}} & {\textit{mixed}}
& {\textit{full}} & {\textit{pruned}} & {\textit{mixed}} \\

\midrule
covtype & 7.7\si{m} & 3.0\si{\times} & 2.3\si{\times} & 5.6\si{h} & 25.4\si{\times} & 6.6\si{\times} & 6.4\si{s} & 1.6\si{\times} & 1.4\si{\times} & 1.1\si{m} & 7.0\si{\times} & 1.6\si{\times} \\
fmnist & 1.3\si{h} & 6.6\si{\times} & 4.7\si{\times} & 4.8\si{h} & 7.6\si{\times} & 7.1\si{\times} & 1.2\si{m} & 1.6\si{\times} & 1.5\si{\times} & 6.9\si{m} & 3.6\si{\times} & 3.1\si{\times} \\
higgs & 3.7\si{h} & 2.8\si{\times} & 1.6\si{\times} & 6.0\si{h}*  & 3.3\si{\times} & 1.0\si{\times} & 1.4\si{m} & 7.5\si{\times} & 1.5\si{\times} & 59.8\si{m} & 13.7\si{\times} & 2.4\si{\times}  \\
miniboone & 2.7\si{h} & 5.1\si{\times} & 3.6\si{\times} & 6.0\si{h}* & 4.5\si{\times} & 1.1\si{\times} & 3.7\si{m} & 15.3\si{\times} & 1.7\si{\times} & 3.0\si{h} & 19.2\si{\times} & 2.0\si{\times}  \\
mnist & 20.0\si{m} & 6.3\si{\times} & 5.5\si{\times} & 2.7\si{h} & 9.8\si{\times} & 5.8\si{\times} & 1.1\si{m} & 2.3\si{\times} & 1.9\si{\times}  & 3.7\si{m} & 2.9\si{\times} & 2.7\si{\times}\\
prostate & 11.8\si{m} & 3.5\si{\times} & 3.0\si{\times} & 6.0\si{h}* & 2.7\si{\times} & 1.3\si{\times} & 12.4\si{s} & 2.3\si{\times} & 2.1\si{\times}  & 23.0\si{m} & 25.2\si{\times} & 2.6\si{\times} \\
roadsafety & 10.3\si{m} & 2.4\si{\times} & 2.0\si{\times} & 5.8\si{h} & 10.2\si{\times} & 2.4\si{\times} & 10.4\si{s} & 2.1\si{\times} & 1.7\si{\times} & 52.2\si{m} & 35.9\si{\times} & 4.5\si{\times} \\
sensorless & 27.1\si{m} & 2.8\si{\times} & 2.4\si{\times} & 6.0\si{h}* & 5.4\si{\times} & 2.9\si{\times} & 12.9\si{s} & 2.2\si{\times} & 1.8\si{\times}  & 3.0\si{m} & 5.4\si{\times} & 1.8\si{\times} \\
vehicle & 2.4\si{h} & 4.4\si{\times} & 3.1\si{\times} & 6.0\si{h}* & 3.8\si{\times} & 1.4\si{\times} & 12.0\si{m} & 19.9\si{\times} & 3.8\si{\times} & 43.6\si{m} & 5.8\si{\times} & 1.0\si{\times}\\
webspam & 23.7\si{m} & 5.6\si{\times} & 4.3\si{\times} & 6.0\si{h}* & 12.2\si{\times} & 7.2\si{\times} & 25.8\si{s} & 2.4\si{\times} & 2.0\si{\times}  & 12.5\si{m} & 3.6\si{\times} & 1.1\si{\times}\\

\bottomrule
\end{tabular}

\end{table*}

Figure~\ref{fig:time_vs_nverified_selection} shows the number of executed searches as a function of time in the four combinations of attack type and model type, for a selected four datasets.\footnote{The supplement shows these plots for all datasets.} For XGB, both attacks benefit. Moreover, the \textit{mixed} setting is typically very close in run time to the \textit{pruned}. On RF, we see that the \textit{pruned} setting offers larger speedups. However, we see a more noticeable difference between it and the \textit{mixed} search on several datasets. This indicates that the \textit{mixed} strategy most fall back more often to an expensive \textit{full} search. 

\begin{figure*}[h!]
    \centering
    %\footnotesize
    \includegraphics{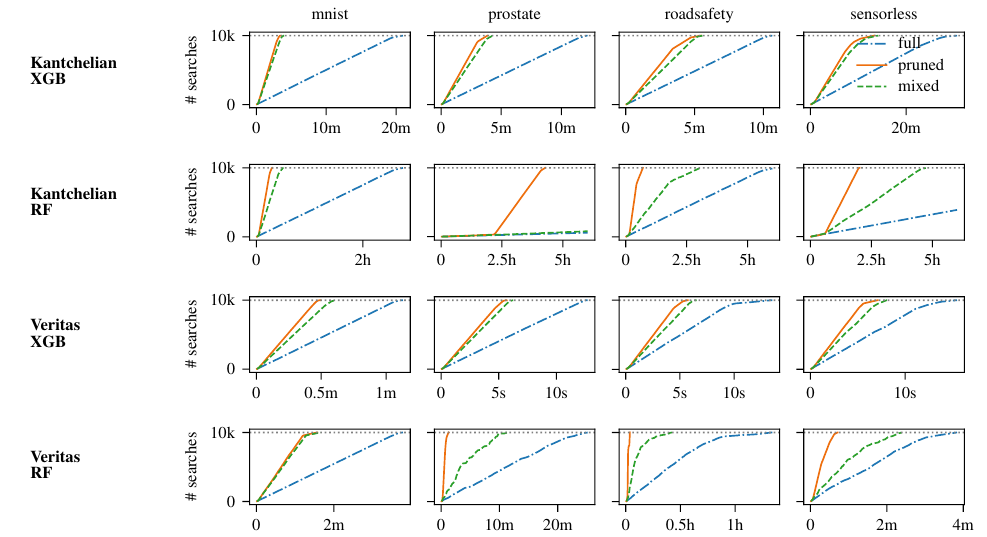}
    \caption{Average run times for 10\,000 calls to \textit{full}, \textit{pruned} and \textit{mixed}) for \textit{kantchelian} (top) and \textit{veritas} (bottom). Results are given for both XGBoost and random forest for four selected datasets.}
    \label{fig:time_vs_nverified_selection}
\end{figure*}

Finally, it is natural to wonder how the quality of the generated adversarial examples is affected by the modified search procedure. While this is difficult to quantify, Figure~\ref{fig:ex_quality} provides some examples of constructed adversarial examples for the \textit{mnist} dataset. Visually, the examples constructed by \textit{full} and \textit{pruned} for both attacks are very similar. The examples constructed using \textit{kantchelian} look more similar to the base example than those for \textit{veritas} because \textit{kantchelian}
 finds the closest possible adversarial example whereas \textit{veritas} has a different objective: it constructs an adversarial example that will explicit a highly confident misprediction. See Appendix~\ref{sec::advex_quality} for more generated examples.

\begin{figure}[h]
    \centering
    \includegraphics[scale=0.8]{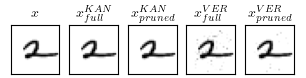}
    \caption{Generated adversarial examples for base example $x$ from \textit{mnist}, with both attacks and both \textit{full} and \textit{pruned} setting.}
    \label{fig:ex_quality}
\end{figure}

\begin{figure*}[h]
    \footnotesize
    \centering
    \includegraphics{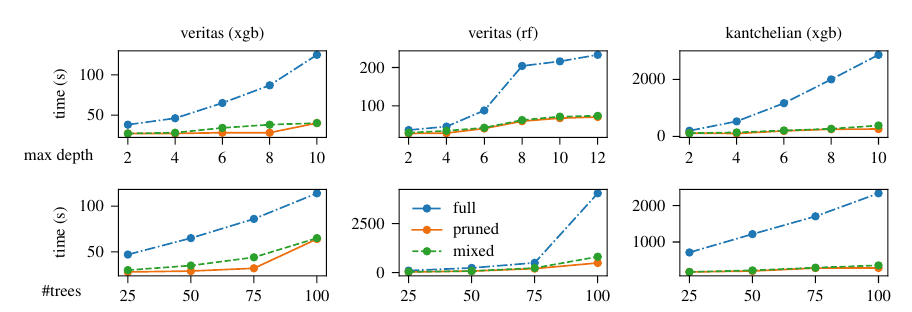}
    \caption{Run time of \textit{full}, \textit{mixed} and \textit{pruned} settings for varying the max depth (top) and number of estimators in the ensemble (bottom).}
    \label{fig:mnist}
\end{figure*}

\subsection{Q2: Scaling Behavior}
\label{sec::hyper_mnist}

Two key hyperparameters of XGB and RF  are the maximum depth of each learned tree and the number of trees in the ensemble. We explore how varying these affects the considered approach employing the same setup as described in Subsection 4.1. We use the \textit{mnist} dataset and omit \textit{kantchelian} with RFs due to its computational cost.

Figure~\ref{fig:mnist} (top) shows how the run time to perform 10,000 searches varies as function of the maximum tree depth for a fixed ensemble size of 50 for XGB and 25 for RF.  The run times for the \textit{pruned} and \textit{mixed} approaches grow very slowly as the depths are increased. In contrast, the \textit{full} search scales worse: deeper trees lead to higher run times.

Figure~\ref{fig:mnist} (bottom) shows how the run time to perform 10,000 searches varies as function of the ensemble size for a fixed maximum tree depth of 6 for XGB and 10 for RF.  Again, the \textit{pruned} and \textit{mixed} approaches show much better scaling behavior. Note that \textit{veritas}'s full search shows a very large jump on RF when moving from 75 to 100 trees.  These results indicate that our approaches will offer even better run time performance than the standard full search for more complex ensembles. 

\subsection{Q3: Empirical FNR}
We use Equation 3 to bound the false negative rate to be less than 25\% with high probability. Table~\ref{tab::res_extended} in the supplement reports the empirical false negative rates for all experiments. The average false negative rate is 6.4\% and the maximum is 11.8\%. Hence, empirically we achieve better results than the theory guarantees. Neither the ensemble method nor the attack type strongly influence the false negative rate.  

Still, we observe that with these false negative rates we can dramatically reduce the number of considered features. On average, $F_S$ contains 19.1\% of the features.  Out of 200 experiments,\footnote{5 folds x 10 datasets x 2 ensemble x 2 attacks} we only select the maximum percentage of features 15 times. Generally, \textit{kantchelian} requires slightly more features than \textit{veritas} and RF models requires slightly more features than XGB models. 

\section{Related Work}
Adversarial examples have been theoretically studied and defined in multiple different ways~\cite{Diochnos2018AdversarialRA,JMLR:v22:20-285}. Approaches to reason about learned tree ensembles have received substantial interest in recent years. These include algorithms for performing evasion attacks~\cite{kantchelian2016evasion,einziger19} (i.e., generate adversarial examples), perform robustness checking~\cite{chen2019robustness}, and verify that the ensembles satisfy certain criteria~\cite{devos21smt,devos21a,ranzato2020,tornblom20}.
Kantchelian et al. \citeyearpar{kantchelian2016evasion} were the first to show that, just like neural networks, tree ensembles are susceptible to evasion attacks. Their MILP formulation is still the most frequently used method to check robustness and generate adversarial examples. Beyond this exact approach, several approximate approaches exist~\cite{chen2019robustness,devos21a,wang20,zhang20} though not all of them are able to generate concrete adversarial examples (e.g., \cite{chen2019robustness,wang20}).

Other work focuses on making tree ensembles more robust. Approaches for this include adding generated adversarial examples to the training data (model hardening) \cite{kantchelian2016evasion}, or modifying the splitting procedure \cite{hchen19robust,calzavara2020treant,vos21a-groot}. 
Gaining further insights into how evasion attacks target tree ensembles, like those contained in this paper, may inspire novel ways to improve the robustness of learners.

\section{Conclusions}
This paper explored two methods to efficiently generate adversarial examples for tree ensembles. We showed that considering only the \emph{same subset of features} is typically sufficient to generate adversarial examples for tree ensemble models. We proposed a simple procedure to quickly identify such a subset of features, and two generic approaches that exploit it to speed up adversarial examples generation. We showed how to apply them to an exact (\textit{kantchelian}) and approximate (\textit{veritas}) evasion attack on tree ensembles, and discussed their properties and run time performances.

\section*{Acknowledgments}
This research is supported by The European Union’s Horizon Europe Research and Innovation program under the grant agreement TUPLES No. 101070149 (LC, LD, OK, JD), the Research Foundation-Flanders (FWO, LD: 1SB1322N), and the Flemish Government under the ``Onderzoeksprogramma Artificële Intelligentie (AI) Vlaanderen'' program (JD).

\section*{Broader Impact Statement}

Machine learning is widely used in many different application areas.
With the wide adoption, machine learned models, including tree ensembles, increasingly become high-stake targets for attackers who might employ evasion attacks to achieve their goal.

While this work proposes ways to speed up attacks, we feel it is important to shed on light on things that adversaries may do. Moreover, insights about possible attacks increase our understanding and may hence yield insights that result in improved defenses or ways to make tree ensembles more robust.

We strongly feel that it is in the interest of the research community that (1) the research community stays on top of these developments so that machine learning libraries can adapt if necessary, and (2) all work done in this area is open-access. For that reason, source codes used in this work will be made available upon acceptance.

\bibliography{main}

\begin{thebibliography}{28}
\providecommand{\natexlab}[1]{#1}
\providecommand{\url}[1]{\texttt{#1}}
\expandafter\ifx\csname urlstyle\endcsname\relax
  \providecommand{\doi}[1]{doi: #1}\else
  \providecommand{\doi}{doi: \begingroup \urlstyle{rm}\Url}\fi

\bibitem[Andriushchenko \& Hein(2019)Andriushchenko and Hein]{andriushchenko19}
Andriushchenko, M. and Hein, M.
\newblock Provably robust boosted decision stumps and trees against adversarial
  attacks.
\newblock In \emph{Advances in Neural Information Processing Systems},
  volume~32, 2019.

\bibitem[Breiman(2001)]{breiman2001random}
Breiman, L.
\newblock Random forests.
\newblock \emph{Machine learning}, 45:\penalty0 5--32, 2001.

\bibitem[Calzavara et~al.(2020)Calzavara, Lucchese, Tolomei, Abebe, and
  Orlando]{calzavara2020treant}
Calzavara, S., Lucchese, C., Tolomei, G., Abebe, S.~A., and Orlando, S.
\newblock Treant: training evasion-aware decision trees.
\newblock \emph{Data Mining and Knowledge Discovery}, 34\penalty0 (5):\penalty0
  1390--1420, 2020.

\bibitem[Chen et~al.(2019{\natexlab{a}})Chen, Zhang, Boning, and
  Hsieh]{hchen19robust}
Chen, H., Zhang, H., Boning, D., and Hsieh, C.-J.
\newblock Robust decision trees against adversarial examples.
\newblock In \emph{International Conference on Machine Learning}, pp.\
  1122--1131, 2019{\natexlab{a}}.

\bibitem[Chen et~al.(2019{\natexlab{b}})Chen, Zhang, Si, Li, Boning, and
  Hsieh]{chen2019robustness}
Chen, H., Zhang, H., Si, S., Li, Y., Boning, D., and Hsieh, C.-J.
\newblock Robustness verification of tree-based models.
\newblock \emph{Advances in Neural Information Processing Systems}, 32,
  2019{\natexlab{b}}.

\bibitem[Chen \& Guestrin(2016)Chen and Guestrin]{chen2016xgboost}
Chen, T. and Guestrin, C.
\newblock {XGBoost}: A scalable tree boosting system.
\newblock In \emph{Proceedings of the 22nd ACM SIGKDD International Conference
  on Knowledge Discovery and Data Mining}, pp.\  785--794, 2016.

\bibitem[Devos et~al.(2021{\natexlab{a}})Devos, Meert, and Davis]{devos21a}
Devos, L., Meert, W., and Davis, J.
\newblock Versatile verification of tree ensembles.
\newblock In \emph{Proceedings of the 38th International Conference on Machine
  Learning}, volume 139 of \emph{Proceedings of Machine Learning Research},
  pp.\  2654--2664, 2021{\natexlab{a}}.

\bibitem[Devos et~al.(2021{\natexlab{b}})Devos, Meert, and Davis]{devos21smt}
Devos, L., Meert, W., and Davis, J.
\newblock Verifying tree ensembles by reasoning about potential instances.
\newblock In \emph{Proceedings of the 2021 SIAM International Conference on
  Data Mining (SDM)}, pp.\  450--458. SIAM, 2021{\natexlab{b}}.
\newblock \doi{10.1137/1.9781611976700.51}.
\newblock URL \url{https://epubs.siam.org/doi/abs/10.1137/1.9781611976700.51}.

\bibitem[Devos et~al.(2024)Devos, Cascioli, and
  Davis]{devos24multiclass-veritas}
Devos, L., Cascioli, L., and Davis, J.
\newblock Robustness verification of multiclass tree ensembles.
\newblock \emph{Proceedings of the AAAI Conference on Artificial Intelligence},
  38:\penalty0 To appear, 2024.

\bibitem[Diochnos et~al.(2018)Diochnos, Mahloujifar, and
  Mahmoody]{Diochnos2018AdversarialRA}
Diochnos, D.~I., Mahloujifar, S., and Mahmoody, M.
\newblock Adversarial risk and robustness: General definitions and implications
  for the uniform distribution.
\newblock In \emph{Neural Information Processing Systems}, 2018.
\newblock URL \url{https://api.semanticscholar.org/CorpusID:53097516}.

\bibitem[Einziger et~al.(2019)Einziger, Goldstein, Sa’ar, and
  Segall]{einziger19}
Einziger, G., Goldstein, M., Sa’ar, Y., and Segall, I.
\newblock Verifying robustness of gradient boosted models.
\newblock \emph{Proceedings of the AAAI Conference on Artificial Intelligence},
  33:\penalty0 2446--2453, 2019.
\newblock \doi{10.1609/aaai.v33i01.33012446}.
\newblock URL \url{https://ojs.aaai.org/index.php/AAAI/article/view/4089}.

\bibitem[Friedman(2001)]{friedman2001greedy}
Friedman, J.~H.
\newblock Greedy function approximation: a gradient boosting machine.
\newblock \emph{Annals of statistics}, pp.\  1189--1232, 2001.

\bibitem[Goodfellow et~al.(2014)Goodfellow, Shlens, and
  Szegedy]{goodfellow2014explaining}
Goodfellow, I.~J., Shlens, J., and Szegedy, C.
\newblock Explaining and harnessing adversarial examples.
\newblock \emph{arXiv preprint arXiv:1412.6572}, 2014.

\bibitem[Gourdeau et~al.(2021)Gourdeau, Kanade, Kwiatkowska, and
  Worrell]{JMLR:v22:20-285}
Gourdeau, P., Kanade, V., Kwiatkowska, M., and Worrell, J.
\newblock On the hardness of robust classification.
\newblock \emph{Journal of Machine Learning Research}, 22\penalty0
  (273):\penalty0 1--29, 2021.
\newblock URL \url{http://jmlr.org/papers/v22/20-285.html}.

\bibitem[Greene \& Wellner(2017)Greene and Wellner]{greene2017exponential}
Greene, E. and Wellner, J.~A.
\newblock Exponential bounds for the hypergeometric distribution.
\newblock \emph{Bernoulli: official journal of the Bernoulli Society for
  Mathematical Statistics and Probability}, 23\penalty0 (3):\penalty0 1911,
  2017.

\bibitem[Guo et~al.(2022)Guo, Teng, Gao, and Zhou]{guo22fast-robust}
Guo, J.-Q., Teng, M.-Z., Gao, W., and Zhou, Z.-H.
\newblock Fast provably robust decision trees and boosting.
\newblock In \emph{Proceedings of the 39th International Conference on Machine
  Learning}, volume 162 of \emph{Proceedings of Machine Learning Research},
  pp.\  8127--8144, 2022.

\bibitem[{Gurobi Optimization, LLC}(2023)]{gurobi}
{Gurobi Optimization, LLC}.
\newblock {Gurobi Optimizer Reference Manual}, 2023.
\newblock URL \url{https://www.gurobi.com}.

\bibitem[Ilyas et~al.(2019)Ilyas, Santurkar, Tsipras, Engstrom, Tran, and
  Madry]{NEURIPS2019_e2c420d9}
Ilyas, A., Santurkar, S., Tsipras, D., Engstrom, L., Tran, B., and Madry, A.
\newblock Adversarial examples are not bugs, they are features.
\newblock In Wallach, H., Larochelle, H., Beygelzimer, A., d\textquotesingle
  Alch\'{e}-Buc, F., Fox, E., and Garnett, R. (eds.), \emph{Advances in Neural
  Information Processing Systems}, volume~32. Curran Associates, Inc., 2019.
\newblock URL
  \url{https://proceedings.neurips.cc/paper_files/paper/2019/file/e2c420d928d4bf8ce0ff2ec19b371514-Paper.pdf}.

\bibitem[Kantchelian et~al.(2016)Kantchelian, Tygar, and
  Joseph]{kantchelian2016evasion}
Kantchelian, A., Tygar, J.~D., and Joseph, A.
\newblock Evasion and hardening of tree ensemble classifiers.
\newblock In \emph{International Conference on Machine Learning}, pp.\
  2387--2396, 2016.

\bibitem[Pedregosa et~al.(2011)Pedregosa, Varoquaux, Gramfort, Michel, Thirion,
  Grisel, Blondel, Prettenhofer, Weiss, Dubourg, Vanderplas, Passos,
  Cournapeau, Brucher, Perrot, and Duchesnay]{scikit-learn}
Pedregosa, F., Varoquaux, G., Gramfort, A., Michel, V., Thirion, B., Grisel,
  O., Blondel, M., Prettenhofer, P., Weiss, R., Dubourg, V., Vanderplas, J.,
  Passos, A., Cournapeau, D., Brucher, M., Perrot, M., and Duchesnay, E.
\newblock Scikit-learn: Machine learning in {P}ython.
\newblock \emph{Journal of Machine Learning Research}, 12:\penalty0 2825--2830,
  2011.

\bibitem[Ranzato \& Zanella(2020)Ranzato and Zanella]{ranzato2020}
Ranzato, F. and Zanella, M.
\newblock Abstract interpretation of decision tree ensemble classifiers.
\newblock In \emph{Proceedings of the AAAI Conference on Artificial
  Intelligence}, volume~34, pp.\  5478--5486, 2020.

\bibitem[Szegedy et~al.(2013)Szegedy, Zaremba, Sutskever, Bruna, Erhan,
  Goodfellow, and Fergus]{szegedy2013intriguing}
Szegedy, C., Zaremba, W., Sutskever, I., Bruna, J., Erhan, D., Goodfellow, I.,
  and Fergus, R.
\newblock Intriguing properties of neural networks.
\newblock \emph{arXiv preprint arXiv:1312.6199}, 2013.

\bibitem[T{\"o}rnblom \& Nadjm-Tehrani(2020)T{\"o}rnblom and
  Nadjm-Tehrani]{tornblom20}
T{\"o}rnblom, J. and Nadjm-Tehrani, S.
\newblock Formal verification of input-output mappings of tree ensembles.
\newblock \emph{Science of Computer Programming}, 194:\penalty0 102450, 2020.

\bibitem[Vos \& Verwer(2021)Vos and Verwer]{vos21a-groot}
Vos, D. and Verwer, S.
\newblock Efficient training of robust decision trees against adversarial
  examples.
\newblock In \emph{Proceedings of the 38th International Conference on Machine
  Learning}, volume 139 of \emph{Proceedings of Machine Learning Research},
  pp.\  10586--10595, 2021.

\bibitem[Vos \& Verwer(2022{\natexlab{a}})Vos and Verwer]{vos22optimal}
Vos, D. and Verwer, S.
\newblock Robust optimal classification trees against adversarial examples.
\newblock In \emph{Proceedings of the AAAI Conference on Artificial
  Intelligence}, volume~36, pp.\  8520--8528, 2022{\natexlab{a}}.

\bibitem[Vos \& Verwer(2022{\natexlab{b}})Vos and Verwer]{vos22relabel}
Vos, D. and Verwer, S.
\newblock Adversarially robust decision tree relabeling.
\newblock In \emph{Joint European Conference on Machine Learning and Knowledge
  Discovery in Databases}, 2022{\natexlab{b}}.

\bibitem[Wang et~al.(2020)Wang, Zhang, Chen, Boning, and Hsieh]{wang20}
Wang, Y., Zhang, H., Chen, H., Boning, D., and Hsieh, C.-J.
\newblock On lp-norm robustness of ensemble decision stumps and trees.
\newblock In \emph{Proceedings of the 37th International Conference on Machine
  Learning}, volume 119 of \emph{Proceedings of Machine Learning Research},
  pp.\  10104--10114. PMLR, 13--18 Jul 2020.
\newblock URL \url{https://proceedings.mlr.press/v119/wang20aa.html}.

\bibitem[Zhang et~al.(2020)Zhang, Zhang, and Hsieh]{zhang20}
Zhang, C., Zhang, H., and Hsieh, C.-J.
\newblock An efficient adversarial attack for tree ensembles.
\newblock In \emph{Advances in Neural Information Processing Systems},
  volume~33, pp.\  16165--16176, 2020.

\end{thebibliography}
\bibliographystyle{icml2024}

\appendix
\section{Analysis of the Problem Setting }%
\label{sec::method:analysis}

Adversarial examples are generated for tasks like computing adversarial accuracy, computing empirical robustness, and performing model hardening. The effect of using the approximation proposed in this paper differs for each task.

Computing the adversarial accuracy of a classifier only requires determining whether an adversarial example $\tilde{x}$ exists within the given $\delta$ for each provided normal example $x$. Because the \textit{mixed} strategy reverts to the original complete search when the \textit{pruned} approach returns an \textit{UNSAT}, as stated in Theorem 3.2 it is guaranteed to find an adversarial example if it exists. Hence, the \textit{mixed} strategy can speed up computing the adversarial accuracy without affecting its value.

Computing the empirical robustness of a classifier requires finding the nearest adversarial example $\tilde{x}$ for each normal example $x$. Because the \textit{pruned} approach does not consider all features and the \textit{mixed} approach may not, they may return an adversarial example that is further away than if the full search space was considered. Hence, when using an exact attack like \textit{kantchelian}, the empirical robustness computed using the \textit{mixed} strategy is an overestimate of the true empirical robustness. We show this and we study what happens with an approximate method in Appendix~\ref{sec::advex_quality}.

In model hardening, a large number of adversarial examples are generated and added to the training data \cite{kantchelian2016evasion}. The \textit{pruned} approach can be used to generate a lot more adversarial examples in a fixed amount of time.

\section{Employed Datasets and Models}
\label{sec::datasets}

Table~\ref{tab::datasets_links} gives specific reference to each of the datasets used in the experiments.  
\begin{table*}[h!]
\footnotesize
\begin{center}
\caption{%Datasets description with number of examples N, number of features \#F and link to access the dataset. 
References to all seven datasets used in the experiments. 
Selected datasets all have more than 50k examples and more than 30 features, so that adversarial example generation is typically time consuming and our methods can speed it up. 
%A subset of the \textit{higgs} dataset is randomly sampled as the full dataset is too big. 
Multiclass classification datasets are reduced to binary classification (classes 0-4 vs. class 5-9 for \textit{mnist} and \textit{fmnist}, classes 0-5 vs. classes 6-10 for \textit{sensorless}).}

\label{tab::datasets_links}

% \begin{tabular}{lrrc}
% \toprule
% \textbf{Dataset} & N & \#F & link \\
% \midrule
% covtype & 581012 & 54 & https://www.openml.org/d/1596 \\
% fmnist & 70000 & 784 & https://www.openml.org/d/40996\\
% higgs & 250000 & 33 & https://www.openml.org/d/42769\\
% miniboone & 72298 & 50 & https://www.openml.org/d/44128\\
% mnist & 70000 & 784 &  https://www.openml.org/d/554 \\
% sensorless & 58509 & 48 & https://archive.ics.uci.edu/dataset/325\\
% webspam & 350000 & 254 & \lorenzo{!!!}  \\
% \bottomrule
% \end{tabular}

\begin{tabular}{ll}
\toprule
\textbf{Dataset}  & link \\
\midrule
covtype &  \url{https://www.openml.org/d/1596} \\
fmnist &  \url{https://www.openml.org/d/40996}\\
higgs &  \url{https://www.openml.org/d/42769}\\
miniboone &  \url{https://www.openml.org/d/44128}\\
mnist &   \url{https://www.openml.org/d/554} \\
prostate &   \url{https://www.openml.org/d/45672} \\
roadsafety &   \url{https://www.openml.org/d/45038} \\
sensorless & \url{https://archive.ics.uci.edu/dataset/325}\\
vehicle &   \url{https://www.openml.org/d/357} \\
webspam &  \url{https://www.csie.ntu.edu.tw/~cjlin/libsvmtools/datasets/binary.html#webspam}  \\
\bottomrule
\end{tabular}

\end{center}
\end{table*}

We tune ensemble-specific hyperparameters through grid search. In both model types, we choose the number of trees in $\{10, 20, 50\}$. Max depth is chosen in the range $[3,6]$ for XGBoost, and in $\{5, 7, 10\}$ for random forest (which typically needs deeper trees to work better). XGBoost learning rate is chosen among $\{0.1, 0.5, 0.9\}$. 
Table~\ref{tab::datasets_hyper} reports the tuned hyperparameters of the learned ensembles after the grid search. When running \textit{kantchelian} on random forests, due to long run times, we had to limit the number of estimators to 25. 

While the model sizes are smaller, these ensembles are already challenging for the full settings of \textit{kantchelian} and \textit{veritas}.  This is also highlighted in Section~\ref{sec::hyper_mnist} where we empirically study the effect of increasing the ensemble size on performance. Those results show that the \textit{full} procedures becomes increasingly slower as the ensemble complexity grows, and our method offers larger wins.

\begin{table}[h]
\footnotesize
\begin{center}
\caption{Learners' tuned hyperparameters after the grid search described in Section~\ref{sec::experiments}. Each ensemble $\bm{T}$ has maximum tree depth d and contains M trees. The learning rate for XGBoost is $\eta$.}
\label{tab::datasets_hyper}

\begin{tabular}{lrrrrr}
    \toprule
    & \multicolumn{3}{c}{\textbf{XGBoost}} & \multicolumn{2}{c}{\textbf{RF}} \\
    \textbf{Dataset} & M & d & $\eta$ & M & d  \\
    \midrule
    covtype & 50 & 6 & 0.9 & 50 & 10 \\
    fmnist & 50 & 6 & 0.1 & 50 & 10\\
    higgs & 50 & 6 & 0.1 & 50 & 10\\
    miniboone & 50 & 6 & 0.1 & 50 & 10 \\
    mnist &  50 & 6 & 0.5 & 50 & 10 \\
    prostate & 50 & 4 & 0.5 & 50 & 10\\
    roadsafety & 50 & 6 & 0.5 & 50 & 10\\
    sensorless &  50 & 6 & 0.5 & 50 & 10 \\
    vehicle & 50 & 6 & 0.1 & 50 & 10\\
    webspam & 50 & 5 & 0.5 & 50 & 10\\
    \bottomrule
\end{tabular}

\end{center}
\end{table}

\section{Expanded Experimental Results}

\subsection{Run time}
\label{sec::runtime_all}

\begin{figure*}[h!]
    \centering
    \includegraphics{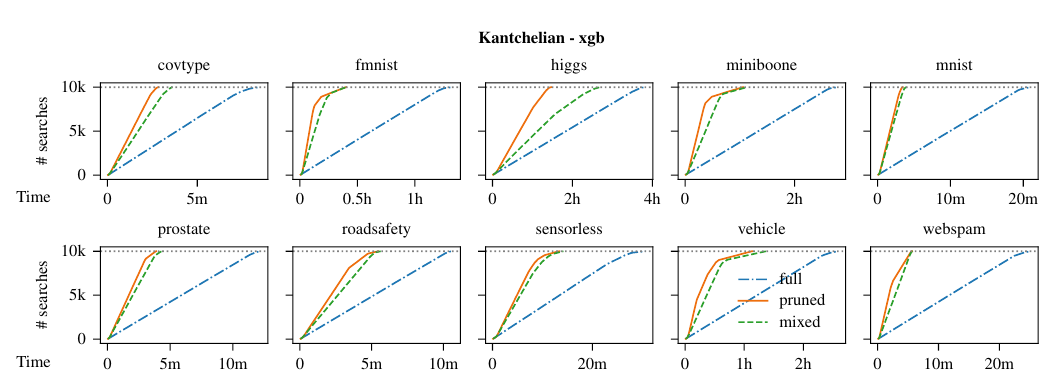}
    \vspace{0.5em}
    \includegraphics{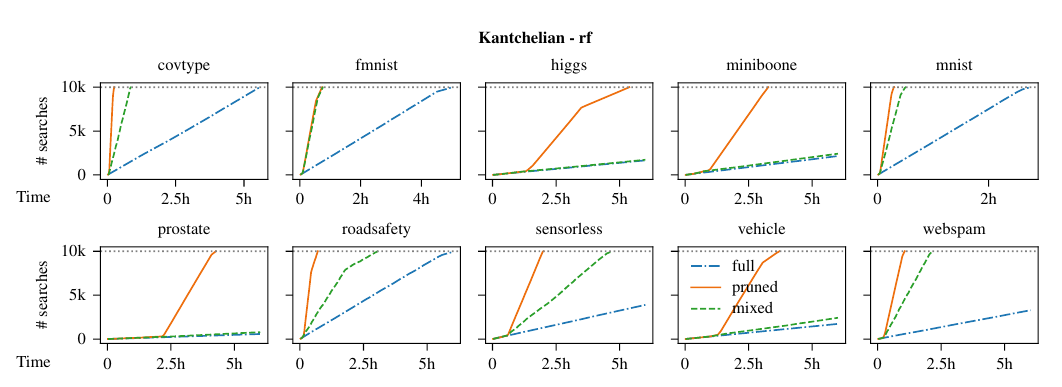}

    \caption{Run times to generate adversarial examples for 10\,000 test examples for the three presented settings (\textit{full}, \textit{pruned} and \textit{mixed}), using \textit{kantchelian} on an XGBoost/random forest ensemble, averaged over 5 folds.} 
    \label{fig:time_vs_nverified_kan}
\end{figure*}

Figure~\ref{fig:time_vs_nverified_veritas} shows the number of executed searches as a function of time when using \textit{veritas} attack on an XGBoost ensemble (top) or a random forest (bottom). 

\begin{figure*}[h!]
    \centering
    \includegraphics{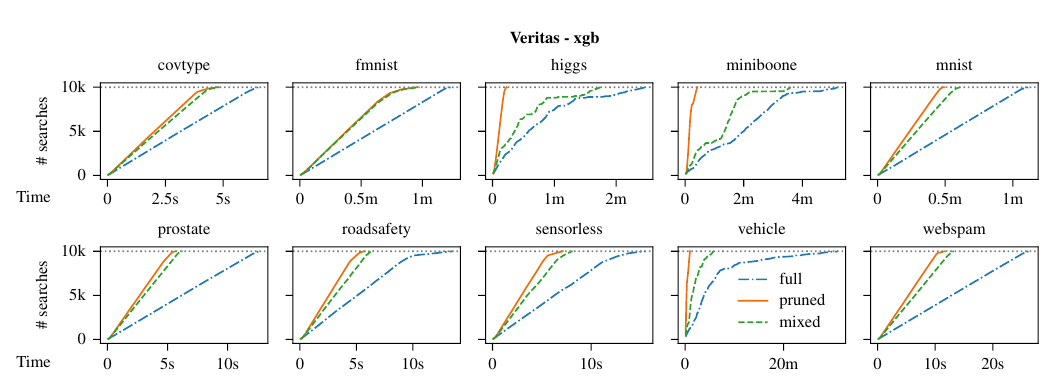}
    \vspace{0.5em}
    \includegraphics{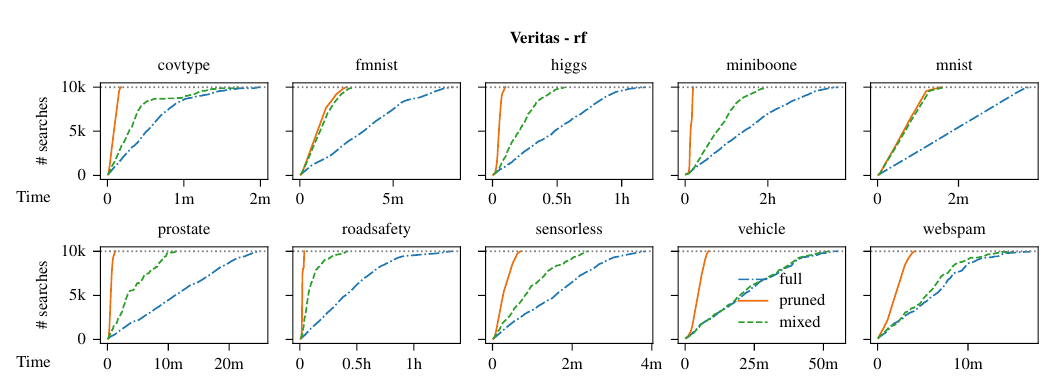}
    
    \caption{Run times to generate adversarial examples for 10\,000 test examples for the three presented settings (\textit{full}, \textit{pruned} and \textit{mixed}), using \textit{veritas} on an XGBoost/random forest ensemble, averaged over 5 folds.} 
    \label{fig:time_vs_nverified_veritas}
\end{figure*}

%Table
Table~\ref{tab::res_extended} shows average time to run 10,000 searches and speedups per dataset.  The averages are computed over five folds. There is one table for each combination of attack (\textit{kantchelian}, \textit{veritas}) and ensemble type (XGB, RF).  
For each dataset, we also report the average size of the relevant feature subset $F_S$, the percent of searches in the \textit{mixed} setting that require making a call to the \textit{full} search, the false negative rate (proportion of times that \textit{pruned} returns UNSAT but \textit{full} returns SAT),
and percent of examples that were skipped due to a method reaching the global timeout of six hours.

%%% KANTCHELIAN - XGB
\begin{table*}
\footnotesize
\begin{center}
%\caption{Average run times to verify 10\,000 test examples using \textit{kantchelian} on an XGBoost ensemble for all three approaches: (\textit{full}, \textit{pruned} and \textit{mixed}).}
%\label{tab::res_kan_xgb}
\caption{Average run times (and speedups) to verify 10\,000 test examples using \textit{kantchelian}/\textit{veritas} on an XGBoost/Random forest ensemble for all three approaches: \textit{full}, \textit{pruned} and \textit{mixed}. For each dataset, we also report the average size of the relevant feature subset, the number of calls to the \textit{full} setting during \textit{mixed} (= number of UNSAT + number of UNK for \textit{pruned}), the number of false negatives (\textit{pruned} returns UNSAT, but \textit{full} returns SAT), and the percent of examples that
were skipped due to a method reaching the global timeout of six hours. Experiments that exceeded such timeout are starred.}
\label{tab::res_extended}
\begin{center}\textbf{Kantchelian, XGBoost}\end{center}
\vspace{0.4em}
\begin{tabular}{@{}lS[table-format=2.1]S[table-format=2.1]S[table-format=2.1]S[table-format=2.1]S[table-format=2.1]S[table-format=2.1]S[table-format=2.1]S[table-format=2.1]S[table-format=2.1]S[table-format=2.1]S[table-format=2.1]@{}}
\toprule
 & {\textit{full}} & {\textit{pruned}} & {} & {\textit{mixed}} & {} & {\% rel. feats} & {\#full calls} & {\#false neg.}  & \textit{full skip} & \textit{prune skip} & \textit{mixed skip}\\
\midrule
covtype & 7.7\si{m} & 2.6\si{m} & 3.0\si{\times} & 3.4\si{m} & 2.3\si{\times} & 5.6\% & 7.5\% & 7.2\% & 0\% & 0\% & 0\% \\
fmnist & 1.3\si{h} & 11.3\si{m} & 6.6\si{\times} & 15.9\si{m} & 4.7\si{\times} & 12.3\% & 6.6\% & 6.6\% & 0\% & 0\% & 0\% \\
higgs & 3.7\si{h} & 1.3\si{h} & 2.8\si{\times} & 2.3\si{h} & 1.6\si{\times} & 21.2\% & 15.5\% & 1.9\% & 0\% & 0\% & 0\% \\
miniboone & 2.7\si{h} & 31.4\si{m} & 5.1\si{\times} & 44.9\si{m} & 3.6\si{\times} & 24.0\% & 11.2\% & 10.1\% & 0\% & 0\% & 0\% \\
mnist & 20.0\si{m} & 3.2\si{m} & 6.3\si{\times} & 3.6\si{m} & 5.5\si{\times} & 10.6\% & 2.1\% & 2.1\% & 0\% & 0\% & 0\% \\
prostate & 11.8\si{m} & 3.4\si{m} & 3.5\si{\times} & 4.0\si{m} & 3.0\si{\times} & 11.6\% & 7.8\% & 7.0\% & 0\% & 0\% & 0\% \\
roadsafety & 10.3\si{m} & 4.4\si{m} & 2.4\si{\times} & 5.2\si{m} & 2.0\si{\times} & 33.8\% & 6.2\% & 6.1\% & 0\% & 0\% & 0\% \\
sensorless & 27.1\si{m} & 9.7\si{m} & 2.8\si{\times} & 11.1\si{m} & 2.4\si{\times} & 31.7\% & 9.5\% & 9.0\% & 0\% & 0\% & 0\% \\
vehicle & 2.4\si{h} & 33.6\si{m} & 4.4\si{\times} & 46.8\si{m} & 3.1\si{\times} & 32.6\% & 10.0\% & 9.4\% & 0\% & 0\% & 0\% \\
webspam & 23.7\si{m} & 4.3\si{m} & 5.6\si{\times} & 5.5\si{m} & 4.3\si{\times} & 8.3\% & 5.5\% & 5.5\% & 0\% & 0\% & 0\% \\

\bottomrule
\end{tabular}
%\end{center}
%\end{table*}

%%% KANTCHELIAN - RF
%\begin{table*}
%\footnotesize
%\begin{center}
%caption{Average run times to verify 10\,000 test examples using \textit{kantchelian} on a random forest ensemble for all three approaches: (\textit{full}, \textit{pruned} and \textit{mixed}).}
%\label{tab::res_ver_xgb}
%\vspace{0.5em}
\begin{center}\textbf{Kantchelian, RF}\end{center}
\begin{tabular}{@{}lS[table-format=2.1]S[table-format=2.1]S[table-format=2.1]S[table-format=2.1]S[table-format=2.1]S[table-format=2.1]S[table-format=2.1]S[table-format=2.1]S[table-format=2.1]S[table-format=2.1]S[table-format=2.1]@{}}
\toprule
 & {\textit{full}} & {\textit{pruned}} & {} & {\textit{mixed}} & {} & {\% rel. feats} & {\#full calls} & {\#false neg.}  & \textit{full skip} & \textit{prune skip} & \textit{mixed skip}\\
\midrule
covtype & 5.6\si{h} & 13.1\si{m} & 25.4\si{\times} & 50.6\si{m} & 6.6\si{\times} & 5.6\% & 11.4\% & 10.7\% & 0\% & 0\% & 0\% \\
fmnist & 4.8\si{h} & 37.7\si{m} & 7.6\si{\times} & 40.3\si{m} & 7.1\si{\times} & 20.3\% & <1\% & <1\% & 0\% & 0\% & 0\% \\
higgs & 6.0\si{h}* & 4.4\si{h} & 3.3\si{\times} & 6.0\si{h} & 1.0\si{\times} & 25.8\% & 81.5\% & 1.4\% & 83.4\% & 0\% & 82.7\% \\
miniboone & 6.0\si{h}* & 3.3\si{h} & 4.5\si{\times} & 6.0\si{h} & 1.1\si{\times} & 34.0\% & 73.7\% & 2.2\% & 78.5\% & 0\% & 75.8\% \\
mnist & 2.7\si{h} & 16.2\si{m} & 9.8\si{\times} & 27.4\si{m} & 5.8\si{\times} & 11.1\% & 8.7\% & 8.7\% & 0\% & 0\% & 0\% \\
prostate & 6.0\si{h}* & 4.2\si{h} & 2.7\si{\times} & 6.0\si{h} & 1.3\si{\times} & 10.8\% & 54.9\% & <1\% & 94.3\% & 0\% & 92.1\% \\
roadsafety & 5.8\si{h} & 34.3\si{m} & 10.2\si{\times} & 2.4\si{h} & 2.4\si{\times} & 26.6\% & 13.9\% & 10.6\% & <1\% & 0\% & 0\% \\
sensorless & 6.0\si{h}* & 2.0\si{h} & 5.4\si{\times} & 4.6\si{h} & 2.9\si{\times} & 37.5\% & 15.3\% & 2.2\% & 61.1\% & 0\% & 0\% \\
vehicle & 6.0\si{h}* & 3.4\si{h} & 3.8\si{\times} & 6.0\si{h} & 1.4\si{\times} & 41.0\% & 54.2\% & 1.9\% & 82.7\% & 0\% & 75.8\% \\
webspam & 6.0\si{h}* & 1.0\si{h} & 12.2\si{\times} & 2.1\si{h} & 7.2\si{\times} & 10.4\% & 6.3\% & 1.8\% & 67.2\% & 0\% & 0\% \\
\bottomrule
\end{tabular}
%\end{center}
%\end{table*}

%%% VERITAS - XGB
%\begin{table*}
%\footnotesize
%\begin{center}
%\caption{Average run times to verify 10\,000 test examples using \textit{veritas} on an XGBoost ensemble for all three approaches: (\textit{full}, \textit{pruned} and \textit{mixed}).}
%\label{tab::res_ver_xgb}
%\vspace{0.5em}
\begin{center}\bf{Veritas, XGBoost}\end{center}
\begin{tabular}{@{}lS[table-format=2.1]S[table-format=2.1]S[table-format=2.1]S[table-format=2.1]S[table-format=2.1]S[table-format=2.1]S[table-format=2.1]S[table-format=2.1]S[table-format=2.1]S[table-format=2.1]S[table-format=2.1]@{}}
\toprule
 & {\textit{full}} & {\textit{pruned}} & {} & {\textit{mixed}} & {} & {\% rel. feats} & {\#full calls} & {\#false neg.} & \textit{full skip} & \textit{prune skip} & \textit{mixed skip}\\
\midrule
covtype & 6.4\si{s} & 4.1\si{s} & 1.6\si{\times} & 4.4\si{s} & 1.4\si{\times} & 7.0\% & 5.4\% & 5.2\% & 0\% & 0\% & 0\% \\
fmnist & 1.2\si{m} & 46.3\si{s} & 1.6\si{\times} & 47.2\si{s} & 1.5\si{\times} & 10.3\% & 1.6\% & 1.4\% & 0\% & 0\% & 0\% \\
higgs & 1.4\si{m} & 11.6\si{s} & 7.5\si{\times} & 58.0\si{s} & 1.5\si{\times} & 21.8\% & 4.5\% & 4.1\% & 0\% & 0\% & 0\% \\
miniboone & 3.7\si{m} & 14.7\si{s} & 15.3\si{\times} & 2.2\si{m} & 1.7\si{\times} & 22.0\% & 9.4\% & 8.4\% & 0\% & 0\% & 0\% \\
mnist & 1.1\si{m} & 28.2\si{s} & 2.3\si{\times} & 34.5\si{s} & 1.9\si{\times} & 5.2\% & 10.6\% & 10.6\% & 0\% & 0\% & 0\% \\
prostate & 12.4\si{s} & 5.3\si{s} & 2.3\si{\times} & 6.0\si{s} & 2.1\si{\times} & 11.2\% & 7.9\% & 7.0\% & 0\% & 0\% & 0\% \\
roadsafety & 10.4\si{s} & 5.0\si{s} & 2.1\si{\times} & 6.2\si{s} & 1.7\si{\times} & 33.8\% & 9.9\% & 9.8\% & 0\% & 0\% & 0\% \\
sensorless & 12.9\si{s} & 5.8\si{s} & 2.2\si{\times} & 7.2\si{s} & 1.8\si{\times} & 21.2\% & 9.8\% & 9.3\% & 0\% & 0\% & 0\% \\
vehicle & 12.0\si{m} & 36.2\si{s} & 19.9\si{\times} & 3.2\si{m} & 3.8\si{\times} & 20.2\% & 12.2\% & 11.8\% & 0\% & 0\% & 0\% \\
webspam & 25.8\si{s} & 10.6\si{s} & 2.4\si{\times} & 12.8\si{s} & 2.0\si{\times} & 5.4\% & 10.1\% & 10.1\% & 0\% & 0\% & 0\% \\
\bottomrule
\end{tabular}
%\end{center}
%\end{table*}

%%% VERITAS - RF
%\begin{table*}
%\footnotesize
%\begin{center}
%\caption{Average run times to verify 10\,000 test examples using \textit{veritas} on a random forest ensemble for all three approaches: (\textit{full}, \textit{pruned} and \textit{mixed}).}
%\label{tab::res_ver_xgb}
%\vspace{0.5em}
\begin{center}\bf{Veritas, RF}\end{center}
\begin{tabular}{@{}lS[table-format=2.1]S[table-format=2.1]S[table-format=2.1]S[table-format=2.1]S[table-format=2.1]S[table-format=2.1]S[table-format=2.1]S[table-format=2.1]S[table-format=2.1]S[table-format=2.1]S[table-format=2.1]@{}}
\toprule
 & {\textit{full}} & {\textit{pruned}} & {} & {\textit{mixed}} & {} & {\% rel. feats} & {\#full calls} & {\#false neg.} & \textit{full skip} & \textit{prune skip} & \textit{mixed skip}\\
\midrule
covtype & 1.1\si{m} & 9.6\si{s} & 7.0\si{\times} & 42.6\si{s} & 1.6\si{\times} & 5.6\% & 11.2\% & 10.4\% & 0\% & 0\% & 0\% \\
fmnist & 6.9\si{m} & 1.9\si{m} & 3.6\si{\times} & 2.2\si{m} & 3.1\si{\times} & 10.9\% & 9.3\% & 5.0\% & 0\% & 0\% & 0\% \\
higgs & 59.8\si{m} & 4.4\si{m} & 13.7\si{\times} & 25.1\si{m} & 2.4\si{\times} & 33.9\% & 13.7\% & 7.4\% & 0\% & 0\% & 0\% \\
miniboone & 3.0\si{h} & 9.3\si{m} & 19.2\si{\times} & 1.5\si{h} & 2.0\si{\times} & 30.8\% & 14.8\% & 10.8\% & 0\% & 0\% & 0\% \\
mnist & 3.7\si{m} & 1.3\si{m} & 2.9\si{\times} & 1.4\si{m} & 2.7\si{\times} & 10.7\% & 4.9\% & 3.4\% & 0\% & 0\% & 0\% \\
prostate & 23.0\si{m} & 54.8\si{s} & 25.2\si{\times} & 8.7\si{m} & 2.6\si{\times} & 11.8\% & 11.0\% & 10.0\% & 0\% & 0\% & 0\% \\
roadsafety & 52.2\si{m} & 1.5\si{m} & 35.9\si{\times} & 11.5\si{m} & 4.5\si{\times} & 25.0\% & 8.7\% & 8.1\% & 0\% & 0\% & 0\% \\
sensorless & 3.0\si{m} & 32.9\si{s} & 5.4\si{\times} & 1.7\si{m} & 1.8\si{\times} & 21.2\% & 10.1\% & 9.7\% & 0\% & 0\% & 0\% \\
vehicle & 43.6\si{m} & 7.5\si{m} & 5.8\si{\times} & 42.1\si{m} & 1.0\si{\times} & 32.6\% & 33.6\% & 6.6\% & 0\% & 0\% & 0\% \\
webspam & 12.5\si{m} & 3.5\si{m} & 3.6\si{\times} & 11.0\si{m} & 1.1\si{\times} & 10.5\% & 12.7\% & 2.0\% & 0\% & 0\% & 0\% \\

\bottomrule
\end{tabular}
\end{center}
\end{table*}

Table~\ref{tab::res_stdev} extends run time results of the presented experiments %presented in Sections~\ref{sec::runtime, sec::runtime_all} 
by also reporting standard deviations.  

Finally, Figures~\ref{fig:time_vs_nverified_kan} and \ref{fig:time_vs_nverified_veritas} show the number of executed searches as a function of time for \textit{kantchelian} and \textit{veritas} on all ten datasets. Each plot contains the results for XGB (top two rows) and RF (bottom two rows).  Hence these plots show the complete set of results from Figure~\ref{fig:time_vs_nverified_selection} in the main paper. 

%%% KANTCHELIAN-XGB

\begin{table}[h!]
\footnotesize
\begin{center}
%\caption{Average run times (with standard deviations) to verify 10\,000 test examples using \textit{kantchelian} for all three approaches: (\textit{full}, \textit{pruned} and \textit{mixed}). The \textit{dtrees} are trained on 500 normal and 500 adversarial examples. }
\caption{Average run times (with standard deviations) to verify 10\,000 test examples using \textit{kantchelian}/\textit{veritas} on an XGBoost/Random forest ensemble for all three approaches: \textit{full}, \textit{pruned} and \textit{mixed}.}
\label{tab::res_stdev}
\vspace{0.4em}
\begin{center}\bf{Kantchelian, XGBoost}\end{center}
\begin{tabular}{l *{3}{r@{\,}c@{\,}l}}
\toprule
 & \multicolumn{3}{c}{\textit{full}} &  \multicolumn{3}{c}{\textit{pruned}} & \multicolumn{3}{c}{\textit{mixed}} \\
\midrule
covtype & 7.7\si{m} & $\pm$ & 33.4\si{s} & 2.6\si{m} & $\pm$ & 10.9\si{s} & 3.4\si{m} & $\pm$ & 14.5\si{s} \\
fmnist & 1.3\si{h} & $\pm$ & 2.7\si{m} & 11.3\si{m} & $\pm$ & 6.5\si{m} & 15.9\si{m} & $\pm$ & 4.6\si{m} \\
higgs & 3.7\si{h} & $\pm$ & 7.4\si{m} & 1.3\si{h} & $\pm$ & 9.7\si{m} & 2.3\si{h} & $\pm$ & 25.1\si{m} \\
miniboone & 2.7\si{h} & $\pm$ & 4.3\si{m} & 31.4\si{m} & $\pm$ & 16.1\si{m} & 44.9\si{m} & $\pm$ & 11.5\si{m} \\
mnist & 20.0\si{m} & $\pm$ & 45.6\si{s} & 3.2\si{m} & $\pm$ & 12.5\si{s} & 3.6\si{m} & $\pm$ & 12.4\si{s} \\
prostate & 11.8\si{m} & $\pm$ & 18.8\si{s} & 3.4\si{m} & $\pm$ & 24.5\si{s} & 4.0\si{m} & $\pm$ & 15.7\si{s} \\
roadsafety & 10.3\si{m} & $\pm$ & 15.7\si{s} & 4.4\si{m} & $\pm$ & 48.4\si{s} & 5.2\si{m} & $\pm$ & 17.0\si{s} \\
sensorless & 27.1\si{m} & $\pm$ & 2.5\si{m} & 9.7\si{m} & $\pm$ & 2.1\si{m} & 11.1\si{m} & $\pm$ & 2.0\si{m} \\
vehicle & 2.4\si{h} & $\pm$ & 6.8\si{m} & 33.6\si{m} & $\pm$ & 18.7\si{m} & 46.8\si{m} & $\pm$ & 17.8\si{m} \\
webspam & 23.7\si{m} & $\pm$ & 1.1\si{m} & 4.3\si{m} & $\pm$ & 1.6\si{m} & 5.5\si{m} & $\pm$ & 9.9\si{s} \\
\bottomrule
\end{tabular}

%%% KANTCHELIAN-RF
\begin{center}\bf{Kantchelian, RF}\end{center}
\begin{tabular}{l *{3}{r@{\,}c@{\,}l}}
\toprule
 & \multicolumn{3}{c}{\textit{full}} &  \multicolumn{3}{c}{\textit{pruned}} & \multicolumn{3}{c}{\textit{mixed}} \\
\midrule
covtype & 5.6\si{h} & $\pm$ & 1.3\si{m} & 13.1\si{m} & $\pm$ & 1.3\si{m} & 50.6\si{m} & $\pm$ & 1.0\si{m} \\
fmnist & 4.8\si{h} & $\pm$ & 15.8\si{m} & 37.7\si{m} & $\pm$ & 6.1\si{m} & 40.3\si{m} & $\pm$ & 5.0\si{m} \\
higgs & 6.0\si{h} & $\pm$ & 1.8\si{s} & 4.4\si{h} & $\pm$ & 56.6\si{m} & 6.0\si{h} & $\pm$ & 4.5\si{s} \\
miniboone & 6.0\si{h} & $\pm$ & 5.3\si{s} & 3.3\si{h} & $\pm$ & 56.5\si{s} & 6.0\si{h} & $\pm$ & 0.7\si{s} \\
mnist & 2.7\si{h} & $\pm$ & 5.0\si{m} & 16.2\si{m} & $\pm$ & 1.2\si{m} & 27.4\si{m} & $\pm$ & 2.6\si{m} \\
prostate & 6.0\si{h} & $\pm$ & 7.9\si{s} & 4.2\si{h} & $\pm$ & 5.0\si{m} & 6.0\si{h} & $\pm$ & 4.2\si{s} \\
roadsafety & 5.8\si{h} & $\pm$ & 13.2\si{m} & 34.3\si{m} & $\pm$ & 7.9\si{m} & 2.4\si{h} & $\pm$ & 38.1\si{m} \\
sensorless & 6.0\si{h} & $\pm$ & 0.2\si{s} & 2.0\si{h} & $\pm$ & 30.5\si{s} & 4.6\si{h} & $\pm$ & 7.3\si{m} \\
vehicle & 6.0\si{h} & $\pm$ & 0.9\si{s} & 3.4\si{h} & $\pm$ & 19.7\si{m} & 6.0\si{h} & $\pm$ & 2.5\si{s} \\
webspam & 6.0\si{h} & $\pm$ & 2.8\si{s} & 1.0\si{h} & $\pm$ & 2.7\si{m} & 2.1\si{h} & $\pm$ & 3.8\si{m} \\
\bottomrule
\end{tabular}

%\end{center}
%\end{table}

%%% VERITAS-XGB
%\begin{table}[h!]
%\footnotesize
%\begin{center}
%\caption{Average run times (with standard deviations) to verify 10\,000 test examples using \textit{veritas}  for all three approaches:  (\textit{full}, \textit{pruned} and \textit{mixed}). The \textit{dtrees} are trained on 500 normal and 500 adversarial examples. }
%\label{tab::res_ver2}
%\begin{tabular}{@{}lS[table-format=4.3]S[table-format=2.1]S[table-format=2.1]@{}}
%\begin{tabular}{lccc}
\begin{center}\bf{Veritas, XGBoost}\end{center}
\begin{tabular}{l *{3}{r@{\,}c@{\,}l}}
\toprule
 & \multicolumn{3}{c}{\textit{full}} &  \multicolumn{3}{c}{\textit{pruned}} & \multicolumn{3}{c}{\textit{mixed}} \\
\midrule
covtype & 6.4\si{s} & $\pm$ & 0.2\si{s} & 4.1\si{s} & $\pm$ & 0.3\si{s} & 4.4\si{s} & $\pm$ & 0.2\si{s} \\
fmnist & 1.2\si{m} & $\pm$ & 1.7\si{s} & 46.3\si{s} & $\pm$ & 6.6\si{s} & 47.2\si{s} & $\pm$ & 6.8\si{s} \\
higgs & 1.4\si{m} & $\pm$ & 32.6\si{s} & 11.6\si{s} & $\pm$ & 1.2\si{s} & 58.0\si{s} & $\pm$ & 24.2\si{s} \\
miniboone & 3.7\si{m} & $\pm$ & 49.4\si{s} & 14.7\si{s} & $\pm$ & 5.2\si{s} & 2.2\si{m} & $\pm$ & 42.0\si{s} \\
mnist & 1.1\si{m} & $\pm$ & 1.8\si{s} & 28.2\si{s} & $\pm$ & 0.8\si{s} & 34.5\si{s} & $\pm$ & 1.7\si{s} \\
prostate & 12.4\si{s} & $\pm$ & 0.2\si{s} & 5.3\si{s} & $\pm$ & 0.3\si{s} & 6.0\si{s} & $\pm$ & 0.2\si{s} \\
roadsafety & 10.4\si{s} & $\pm$ & 1.6\si{s} & 5.0\si{s} & $\pm$ & 0.5\si{s} & 6.2\si{s} & $\pm$ & 0.2\si{s} \\
sensorless & 12.9\si{s} & $\pm$ & 1.6\si{s} & 5.8\si{s} & $\pm$ & 0.7\si{s} & 7.2\si{s} & $\pm$ & 0.7\si{s} \\
vehicle & 12.0\si{m} & $\pm$ & 9.7\si{m} & 36.2\si{s} & $\pm$ & 18.6\si{s} & 3.2\si{m} & $\pm$ & 1.6\si{m} \\
webspam & 25.8\si{s} & $\pm$ & 0.6\si{s} & 10.6\si{s} & $\pm$ & 0.5\si{s} & 12.8\si{s} & $\pm$ & 0.5\si{s} \\
\bottomrule
\end{tabular}

%%% VERITAS-RF
\begin{center}\bf{Veritas, RF}\end{center}
\begin{tabular}{l *{3}{r@{\,}c@{\,}l}}
\toprule
 & \multicolumn{3}{c}{\textit{full}} &  \multicolumn{3}{c}{\textit{pruned}} & \multicolumn{3}{c}{\textit{mixed}} \\
\midrule
covtype & 1.1\si{m} & $\pm$ & 27.2\si{s} & 9.6\si{s} & $\pm$ & 0.5\si{s} & 42.6\si{s} & $\pm$ & 29.9\si{s} \\
fmnist & 6.9\si{m} & $\pm$ & 1.1\si{m} & 1.9\si{m} & $\pm$ & 27.6\si{s} & 2.2\si{m} & $\pm$ & 30.6\si{s} \\
higgs & 59.8\si{m} & $\pm$ & 6.7\si{m} & 4.4\si{m} & $\pm$ & 54.0\si{s} & 25.1\si{m} & $\pm$ & 5.4\si{m} \\
miniboone & 3.0\si{h} & $\pm$ & 37.1\si{m} & 9.3\si{m} & $\pm$ & 1.7\si{m} & 1.5\si{h} & $\pm$ & 19.2\si{m} \\
mnist & 3.7\si{m} & $\pm$ & 5.2\si{s} & 1.3\si{m} & $\pm$ & 8.6\si{s} & 1.4\si{m} & $\pm$ & 9.0\si{s} \\
prostate & 23.0\si{m} & $\pm$ & 2.0\si{m} & 54.8\si{s} & $\pm$ & 12.1\si{s} & 8.7\si{m} & $\pm$ & 2.1\si{m} \\
roadsafety & 52.2\si{m} & $\pm$ & 15.9\si{m} & 1.5\si{m} & $\pm$ & 22.2\si{s} & 11.5\si{m} & $\pm$ & 7.3\si{m} \\
sensorless & 3.0\si{m} & $\pm$ & 31.8\si{s} & 32.9\si{s} & $\pm$ & 8.8\si{s} & 1.7\si{m} & $\pm$ & 30.7\si{s} \\
vehicle & 43.6\si{m} & $\pm$ & 10.1\si{m} & 7.5\si{m} & $\pm$ & 51.5\si{s} & 42.1\si{m} & $\pm$ & 10.4\si{m} \\
webspam & 12.5\si{m} & $\pm$ & 2.8\si{m} & 3.5\si{m} & $\pm$ & 27.6\si{s} & 11.0\si{m} & $\pm$ & 2.8\si{m} \\
\bottomrule
\end{tabular}

\end{center}
\end{table}

\subsection{Timeouts}
\label{sec::timeouts}
Table~\ref{tab::timeouts} completes the discussion by showing the percentage of searches that timed out for each dataset, ensemble type and method.
In short, XGBoost ensembles are on average easier to verify, and the searches almost never time out. On the other hand, random forests are more challenging. It can happen that with a strict timeout, the \textit{pruned} setting is not able to find a solution, as the task remains complex even working with a reduced feature set. In those cases, \textit{pruned} ends with a TIMEOUT and \textit{mixed} will have to execute the full search.
In particular for \textit{kantchelian} on random forest, datasets with a lot of \textit{pruned} timeouts are those that then hit the six hours global timeout. This is coherent with the discussion from Section~\ref{sec::runtime}.

\begin{table}[h!]
\footnotesize
\begin{center}
\caption{Average fraction of timeouts incurred when verifying 10\,000 test examples using \textit{kantchelian}/\textit{veritas} on XGBoost/Random forest for all three approaches: \textit{full}, \textit{pruned} and \textit{mixed}.}
\label{tab::timeouts}
\vspace{1em}
\begin{center}\bf{Kantchelian, XGBoost}\end{center}
\begin{tabular}{lccc}
\toprule
  & \textit{full} & \textit{pruned} & \textit{mixed} \\
\midrule
covtype & 0\% & 0\% & 0\% \\
fmnist & 0\% & $<$1\% & 0\% \\
higgs & 0\% & 13.2\% & 0\% \\
miniboone & 0\% & $<$1\% & 0\% \\
mnist & 0\% & 0\% & 0\% \\
prostate & 0\% & 0\% & 0\% \\
roadsafety & 0\% & 0\% & 0\% \\
sensorless & 0\% & $<$1\% & 0\% \\
vehicle & 0\% & $<$1\% & 0\% \\
webspam & 0\% & 0\% & 0\% \\
\bottomrule
\end{tabular}

%\end{center}
%\end{table}

\begin{center}\bf{Kantchelian, RF}\end{center}
%\begin{tabular}{@{}lS[table-format=4.3]S[table-format=2.1]S[table-format=2.1]@{}}
\begin{tabular}{lccc}
%\begin{tabular}{l *{3}{r@{\,}c@{\,}l}}
\toprule
 %& \multicolumn{3}{c}{\textit{full}} &  \multicolumn{3}{c}{\textit{pruned}} & \multicolumn{3}{c}{\textit{mixed}} \\
  & \textit{full} & \textit{pruned} & \textit{mixed} \\
\midrule
covtype & 0\% & $<$1\% & 0\% \\
fmnist & 0\% & $<$1\% & 0\% \\
higgs & $<$1\% & 69.6\% & $<$1\% \\
miniboone & 0\% & 59.8\% & 0\% \\
mnist & $<$1\% & 0\% & 0\% \\
prostate & $<$1\% & 43.5\% & $<$1\% \\
roadsafety & 0\% & 2.4\% & 0\% \\
sensorless & 0\% & 9.3\% & 0\% \\
vehicle & 0\% & 38.4\% & 0\% \\
webspam & 0\% & $<$1\% & 0\% \\
\bottomrule
\end{tabular}

%\begin{table}[h!]
%\footnotesize
%\begin{center}
%\caption{Average fraction of timeouts incurred when verifying 10\,000 test examples using \textit{veritas} for all three approaches:  (\textit{full}, \textit{pruned} and \textit{mixed}). The \textit{dtrees} are trained on 500 normal and 500 adversarial examples. }
%\label{tab::timeouts_ver}

\begin{center}\bf{Veritas, XGBoost}\end{center}
%\begin{tabular}{@{}lS[table-format=4.3]S[table-format=2.1]S[table-format=2.1]@{}}
\begin{tabular}{lccc}
%\begin{tabular}{l *{3}{r@{\,}c@{\,}l}}
\toprule
 %& \multicolumn{3}{c}{\textit{full}} &  \multicolumn{3}{c}{\textit{pruned}} & \multicolumn{3}{c}{\textit{mixed}} \\
  & \textit{full} & \textit{pruned} & \textit{mixed} \\
\midrule
covtype & 0\% & 0\% & 0\% \\
fmnist & 0\% & $<$1\% & 0\% \\
higgs & 0\% & 0\% & 0\% \\
miniboone & $<$1\% & 0\% & $<$1\% \\
mnist & 0\% & 0\% & 0\% \\
prostate & 0\% & 0\% & 0\% \\
roadsafety & 0\% & 0\% & 0\% \\
sensorless & 0\% & $<$1\% & 0\% \\
vehicle & $<$1\% & $<$1\% & $<$1\% \\
webspam & 0\% & 0\% & 0\% \\

\bottomrule
\end{tabular}

\begin{center}\bf{Veritas, RF}\end{center}
%\begin{tabular}{@{}lS[table-format=4.3]S[table-format=2.1]S[table-format=2.1]@{}}
\begin{tabular}{lccc}
%\begin{tabular}{l *{3}{r@{\,}c@{\,}l}}
\toprule
 %& \multicolumn{3}{c}{\textit{full}} &  \multicolumn{3}{c}{\textit{pruned}} & \multicolumn{3}{c}{\textit{mixed}} \\
  & \textit{full} & \textit{pruned} & \textit{mixed} \\
\midrule
covtype & 0\% & 0\% & 0\% \\
fmnist & $<$1\% & 4.3\% & 0\% \\
higgs & $<$1\% & 5.2\% & $<$1\% \\
miniboone & $<$1\% & 2.7\% & $<$1\% \\
mnist & 0\% & 1.5\% & 0\% \\
prostate & $<$1\% & $<$1\% & $<$1\% \\
roadsafety & $<$1\% & $<$1\% & $<$1\% \\
sensorless & $<$1\% & $<$1\% & $<$1\% \\
vehicle & $<$1\% & 26.9\% & $<$1\% \\
webspam & $<$1\% & 10.6\% & $<$1\% \\
\bottomrule
\end{tabular}

\end{center}
\end{table}

\section{Quality of Generated Adversarial Examples}
\label{sec::advex_quality}

We extend Figure~\ref{fig:ex_quality} by further discussing the quality of generated adversarial examples providing more examples, and looking in detail at their distance with respect to the related base example. 

Figure~\ref{fig:advex_quality_big} shows a large set of adversarial examples generated for a \textit{mnist} digit using \textit{kantchelian} and \textit{veritas}.
For each attack, we plot the base example $x$ and the two adversarial examples generated with the \textit{full} and the \textit{pruned} setting. 

\begin{figure*}[h]
    \centering
    \includegraphics{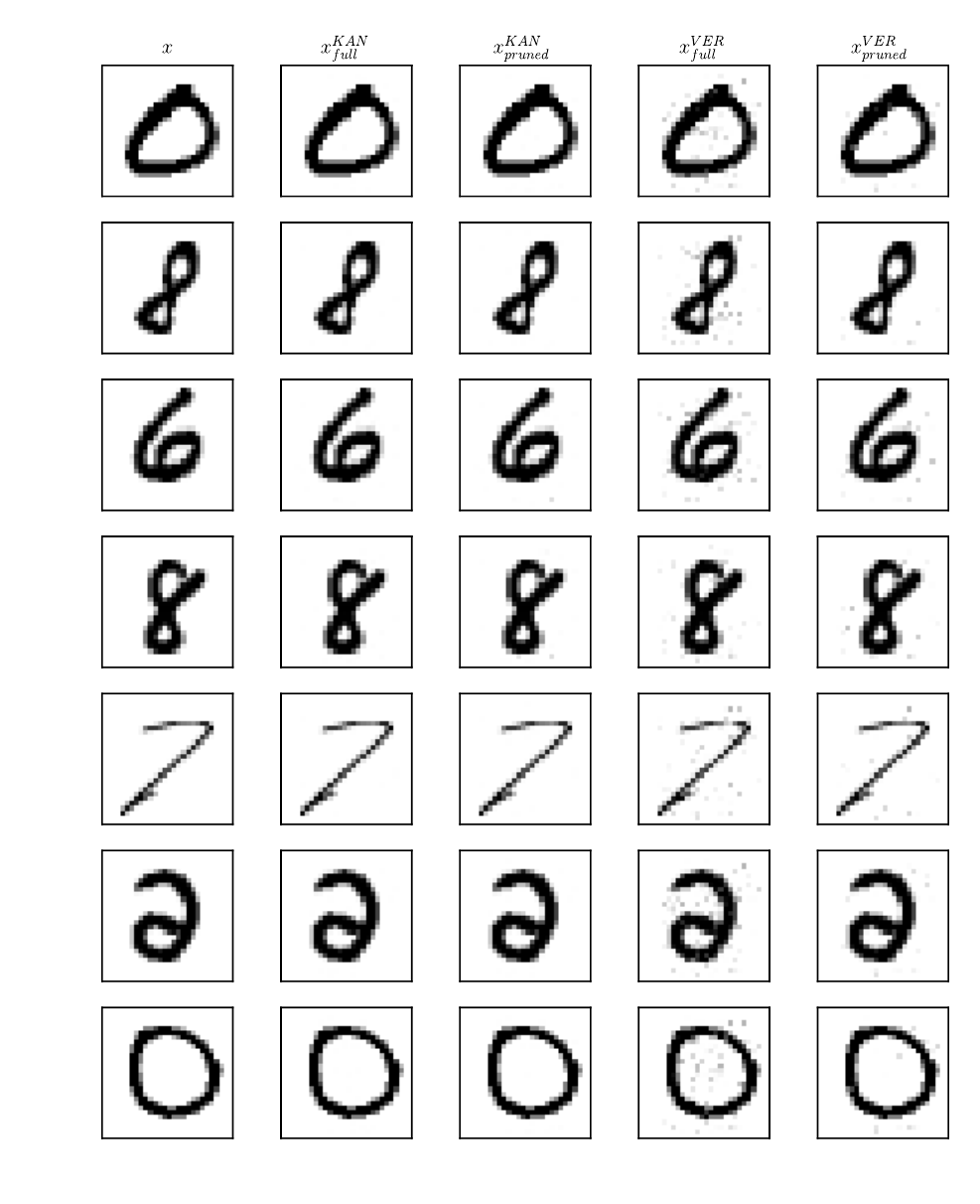}
   
    \caption{Adversarial examples generated for \textit{mnist} with both attacks (\textit{kantchelian} and \textit{veritas}) to show the quality of generated examples.} 
    \label{fig:advex_quality_big}
\end{figure*}

\subsection{Empirical Robustness}

Table~\ref{tab::emprob} shows the average empirical robustness for all experiments for the \textit{full}, \textit{pruned} and \textit{mixed} settings.
\textit{Empirical robustness} is defined as the average distance to the nearest adversarial example for each $x$ in our test set. We use adversarial examples generated with the experiments presented in Section~\ref{sec::runtime}. 

The objective of the \textit{kantchelian} attack is to find the closest adversarial example.  Given that the method is exact, the \textit{full} setting returns the optimal solution. 
The \textit{pruned} search works with a restricted feature set, thus it might not be able to find the closest adversarial example, if that requires altering features not included in the selected feature subset. As a consequence, the \textit{empirical robustness} values for the \textit{pruned} and \textit{mixed} search are overestimates of the true value given by the \textit{full} setting.

Unlike \textit{kantchelian}, \textit{veritas} does not try to find the closest adversarial example.  Instead, it maximizes the confidence that the ensemble assigns to the incorrect label. In this case, there is little difference in the empirical robustness values among all considered settings, with the \textit{pruned} and \textit{mixed} settings typically managing to even lower the distance to the base example. 
 
\begin{table}[h!]
\footnotesize
\begin{center}
\caption{Average empirical robustness (i.e., distance to the closest adversarial example) for the \textit{full}, \textit{mixed} and \textit{pruned} methods using \textit{kantchelian}/\textit{veritas} attacks on XGBoost/RF ensembles.}
\label{tab::emprob}
\vspace{1em}
\begin{center}\bf{Kantchelian, XGBoost}\end{center}
\begin{tabular}{l *{3}{r@{\,}c@{\,}l}}
\toprule
& \multicolumn{3}{c}{\textit{full}} &  \multicolumn{3}{c}{\textit{pruned}} & \multicolumn{3}{c}{\textit{mixed}} \\
\midrule
covtype & 0.016 & $\pm$ & 0.001 & 0.036 & $\pm$ & 0.0 & 0.037 & $\pm$ & 0.0 \\
fmnist & 0.031 & $\pm$ & 0.003 & 0.071 & $\pm$ & 0.011 & 0.07 & $\pm$ & 0.01 \\
higgs & 0.011 & $\pm$ & 0.0 & 0.015 & $\pm$ & 0.001 & 0.015 & $\pm$ & 0.001 \\
miniboone & 0.023 & $\pm$ & 0.0 & 0.031 & $\pm$ & 0.002 & 0.033 & $\pm$ & 0.003 \\
mnist & 0.008 & $\pm$ & 0.001 & 0.03 & $\pm$ & 0.006 & 0.029 & $\pm$ & 0.005 \\
prostate & 0.02 & $\pm$ & 0.0 & 0.037 & $\pm$ & 0.002 & 0.039 & $\pm$ & 0.002 \\
roadsafety & 0.005 & $\pm$ & 0.0 & 0.013 & $\pm$ & 0.002 & 0.013 & $\pm$ & 0.002 \\
sensorless & 0.009 & $\pm$ & 0.0 & 0.014 & $\pm$ & 0.001 & 0.015 & $\pm$ & 0.001 \\
vehicle & 0.016 & $\pm$ & 0.0 & 0.038 & $\pm$ & 0.011 & 0.037 & $\pm$ & 0.009 \\
webspam & 0.002 & $\pm$ & 0.0 & 0.005 & $\pm$ & 0.002 & 0.005 & $\pm$ & 0.002 \\
\bottomrule
\end{tabular}

\begin{center}\bf{Kantchelian, RF}\end{center}
\begin{tabular}{l *{3}{r@{\,}c@{\,}l}}
\toprule
& \multicolumn{3}{c}{\textit{full}} &  \multicolumn{3}{c}{\textit{pruned}} & \multicolumn{3}{c}{\textit{mixed}} \\
\midrule
covtype & 0.076 & $\pm$ & 0.0 & 0.089 & $\pm$ & 0.0 & 0.093 & $\pm$ & 0.0 \\
fmnist & 0.018 & $\pm$ & 0.001 & 0.039 & $\pm$ & 0.003 & 0.039 & $\pm$ & 0.002 \\
higgs & 0.016 & $\pm$ & 0.0 & 0.016 & $\pm$ & 0.005 & 0.017 & $\pm$ & 0.001 \\
miniboone & 0.027 & $\pm$ & 0.001 & 0.035 & $\pm$ & 0.0 & 0.03 & $\pm$ & 0.001 \\
mnist & 0.006 & $\pm$ & 0.0 & 0.035 & $\pm$ & 0.002 & 0.034 & $\pm$ & 0.003 \\
prostate & 0.048 & $\pm$ & 0.001 & 0.074 & $\pm$ & 0.0 & 0.067 & $\pm$ & 0.001 \\
roadsafety & 0.019 & $\pm$ & 0.001 & 0.023 & $\pm$ & 0.0 & 0.026 & $\pm$ & 0.0 \\
sensorless & 0.014 & $\pm$ & 0.0 & 0.022 & $\pm$ & 0.001 & 0.022 & $\pm$ & 0.001 \\
vehicle & 0.016 & $\pm$ & 0.0 & 0.025 & $\pm$ & 0.004 & 0.022 & $\pm$ & 0.001 \\
webspam & 0.003 & $\pm$ & 0.0 & 0.006 & $\pm$ & 0.0 & 0.007 & $\pm$ & 0.0  \\
\bottomrule
\end{tabular}

\begin{center}\bf{Veritas, XGBoost}\end{center}
\begin{tabular}{l *{3}{r@{\,}c@{\,}l}}
\toprule
& \multicolumn{3}{c}{\textit{full}} &  \multicolumn{3}{c}{\textit{pruned}} & \multicolumn{3}{c}{\textit{mixed}} \\
\midrule

covtype & 0.094 & $\pm$ & 0.0 & 0.088 & $\pm$ & 0.002 & 0.089 & $\pm$ & 0.002 \\
fmnist & 0.291 & $\pm$ & 0.002 & 0.282 & $\pm$ & 0.005 & 0.282 & $\pm$ & 0.005 \\
higgs & 0.074 & $\pm$ & 0.001 & 0.069 & $\pm$ & 0.001 & 0.069 & $\pm$ & 0.001 \\
miniboone & 0.077 & $\pm$ & 0.0 & 0.075 & $\pm$ & 0.0 & 0.076 & $\pm$ & 0.0 \\
mnist & 0.291 & $\pm$ & 0.001 & 0.265 & $\pm$ & 0.011 & 0.268 & $\pm$ & 0.01 \\
prostate & 0.097 & $\pm$ & 0.0 & 0.095 & $\pm$ & 0.001 & 0.095 & $\pm$ & 0.0 \\
roadsafety & 0.057 & $\pm$ & 0.0 & 0.056 & $\pm$ & 0.001 & 0.056 & $\pm$ & 0.001 \\
sensorless & 0.056 & $\pm$ & 0.0 & 0.052 & $\pm$ & 0.002 & 0.052 & $\pm$ & 0.002 \\
vehicle & 0.14 & $\pm$ & 0.001 & 0.133 & $\pm$ & 0.003 & 0.134 & $\pm$ & 0.002 \\
webspam & 0.039 & $\pm$ & 0.0 & 0.035 & $\pm$ & 0.002 & 0.035 & $\pm$ & 0.001 \\

\bottomrule
\end{tabular}

\begin{center}\bf{Veritas, RF}\end{center}
\begin{tabular}{l *{3}{r@{\,}c@{\,}l}}
\toprule
& \multicolumn{3}{c}{\textit{full}} &  \multicolumn{3}{c}{\textit{pruned}} & \multicolumn{3}{c}{\textit{mixed}} \\
\midrule
covtype & 0.271 & $\pm$ & 0.001 & 0.24 & $\pm$ & 0.006 & 0.243 & $\pm$ & 0.006 \\
fmnist & 0.294 & $\pm$ & 0.001 & 0.282 & $\pm$ & 0.004 & 0.283 & $\pm$ & 0.003 \\
higgs & 0.072 & $\pm$ & 0.0 & 0.07 & $\pm$ & 0.001 & 0.07 & $\pm$ & 0.0 \\
miniboone & 0.078 & $\pm$ & 0.0 & 0.077 & $\pm$ & 0.0 & 0.077 & $\pm$ & 0.0 \\
mnist & 0.29 & $\pm$ & 0.001 & 0.275 & $\pm$ & 0.005 & 0.276 & $\pm$ & 0.005 \\
prostate & 0.194 & $\pm$ & 0.0 & 0.186 & $\pm$ & 0.001 & 0.187 & $\pm$ & 0.001 \\
roadsafety & 0.11 & $\pm$ & 0.0 & 0.099 & $\pm$ & 0.005 & 0.1 & $\pm$ & 0.004 \\
sensorless & 0.112 & $\pm$ & 0.0 & 0.103 & $\pm$ & 0.003 & 0.104 & $\pm$ & 0.003 \\
vehicle & 0.139 & $\pm$ & 0.001 & 0.132 & $\pm$ & 0.002 & 0.136 & $\pm$ & 0.001 \\
webspam & 0.058 & $\pm$ & 0.0 & 0.053 & $\pm$ & 0.001 & 0.054 & $\pm$ & 0.001 \\

\bottomrule
\end{tabular}

\end{center}
\end{table}

\subsection{Change in Predicted Probability for Adversarial Examples}

\textit{veritas} tries to generate an adversarial example such that the ensemble assigns as a high a probability as possible to the incorrect label. Hence, a natural empirical measure for the quality of the examples generated is to compare the difference in the ensembles probabilistic predictions for the adversarial examples generated by each approach. Namely, we compute $T(\tilde{x}$) - $T(\tilde{x}')$ where $\tilde{x}$ is generated by the full search,  $\tilde{x}'$ is generated by the \textit{pruned} (\textit{mixed}) search, and (in an abuse of notation) $T(x)$ returns the probability an example belongs to most likely class.

Table~\ref{tab::diffprob_full_pruned} shows the average differences in predicted probability between \textit{full} and \textit{pruned}/\textit{mixed} adversarial examples.

Using \textit{kantchelian}, adversarial examples generated with our approaches are assigned very similar probabilities to those generated with the \textit{full} search. In \textit{veritas}, differences are typically higher, as the model output is directly optimized. 

%%% KANTCHELIAN - XGB
\begin{table*}[h]
\footnotesize
\centering

\caption{Average difference in predicted probability between an adversarial example generated with the \textit{full} setting and an adversarial example generated with the \textit{pruned}/\textit{mixed} setting, for the same base example.}
\vspace{1em}
\label{tab::diffprob_full_pruned}
%\begin{center}\bf{Kantchelian, XGBoost}\end{center}
%\begin{tabular}{@{}lS[table-format=2.1]S[table-format=2.1]S[table-format=2.1]S[table-format=2.1]S[table-format=2.1]S[table-format=2.1]S[table-format=2.1]S[table-format=2.1]S[table-format=2.1]S[table-format=2.1]S[table-format=2.1]S[table-format=2.1]@{}}
\begin{tabular}{@{}lrrrrrrrr@{}}
\toprule
 %& \multicolumn{6}{c}{\em kantchelian} & \multicolumn{6}{c}{\em veritas} \\
%\cmidrule(lr){2-7}
%\cmidrule(lr){8-13}

& \multicolumn{2}{c}{\textbf{Kantchelian XGB}}
& \multicolumn{2}{c}{\textbf{Kantchelian RF}}
& \multicolumn{2}{c}{\textbf{Veritas XGB}}
& \multicolumn{2}{c}{\textbf{Veritas RF}} \\
\cmidrule(lr){2-3}
\cmidrule(lr){4-5}
\cmidrule(lr){6-7}
\cmidrule(lr){8-9}
& {\textit{pruned}} & {\textit{mixed}}
& {\textit{pruned}} & {\textit{mixed}} 
& {\textit{pruned}} & {\textit{mixed}}
& {\textit{pruned}} & {\textit{mixed}} \\

\midrule
covtype 
& 0.097 & 0.090
& 0.036 & 0.032  
& 0.106 & 0.100 
& 0.062 & 0.056 
\\

fmnist 
& 0.018 & 0.017 
& 0.045 & 0.045
& 0.319 & 0.314
& 0.376 & 0.341 
\\

higgs 
& 0.010  & 0.008 
& 0.016 & 0.006 
& 0.094 & 0.090 
& 0.053  & 0.046 
\\
miniboone
& 0.014 & 0.013 
& 0.033 & 0.014 
& 0.135 & 0.124 
& 0.086 & 0.075 
\\
mnist 
& 0.109 & 0.107 
& 0.045 & 0.041 
& 0.172 & 0.154 
& 0.29  & 0.276 
\\
prostate
& 0.026 & 0.024 
& 0.034 & 0.026 
& 0.231 & 0.213
& 0.230  & 0.206 
\\
roadsafety
& 0.129 & 0.12 
& 0.023 & 0.020 
& 0.178 & 0.160 
& 0.082 & 0.075 
\\
sensorless
& 0.058 & 0.052 
& 0.053 & 0.045 
& 0.122 & 0.110
& 0.148 & 0.133 
\\
vehicle
& 0.013 & 0.012 
& 0.024 & 0.014 
& 0.200 & 0.175 
& 0.153 & 0.101 
\\
webspam
& 0.047  & 0.044 
& 0.032 & 0.031 
& 0.273 & 0.245 
& 0.262 & 0.229 
\\
\bottomrule
\end{tabular}

\end{table*}

\section{Expanded Related Work}
\label{sec::exp_relwork}

Adversarial examples have been theoretically studied and defined in multiple different ways~\cite{Diochnos2018AdversarialRA,JMLR:v22:20-285}. More specifically, Ilyas et al. showed how certain features in a dataset might be fragile and thus naturally lead to adversarial examples~\cite{NEURIPS2019_e2c420d9}. 
%Adversarial example generation for and robustness checking of decision tree ensembles, and more generally tree ensemble verification, has received tangible attention in the last few years.
Approaches to reason about learned tree ensembles have received substantial interest in recent years. These include algorithms for performing evasion attacks~\cite{kantchelian2016evasion,einziger19} (i.e., generate adversarial examples), perform robustness checking~\cite{chen2019robustness}, and verify that the ensembles satisfy certain criteria~\cite{devos21smt,devos21a,ranzato2020,tornblom20}.
Kantchelian et al. \cite{kantchelian2016evasion} were the first to show that, just like neural networks, tree ensembles are susceptible to evasion attacks. Their MILP formulation is still the most frequently used method to check robustness and generate adversarial examples.
Other notable methods for adversarial example generation are SMT-based systems \cite{einziger19,devos21smt}. These approaches propose varying ways to encode a tree ensemble in a set of logical formulas using the primitives from Satisfiability Modulo Theories (SMT). While the formulation of an ensemble in SMT is very elegant, it tends to perform worse than MILP in practice.

Because MILP and SMT are exact approaches,\footnote{MILP is technically anytime, but the approximate solutions are not useful in practice for this problem setting, see \cite{devos21a}.} they search for the optimal answer which in certain cases can be difficult (i.e., time consuming) to find. Often an approximate answer will be sufficient and several approximate methods have been proposed that are specifically tailored to tree ensembles. Chen et al. proposed a $K$-partite graph representation in which a max-clique corresponds to a specific output of the ensemble \cite{chen2019robustness,wang20}. They introduced a fast method to approximately evaluate robustness, but it cannot generate concrete adversarial examples. Devos et al. further improved upon this work by proposing a heuristic search procedure in this graph which is capable of finding concrete adversarial examples very effectively \cite{devos21a}.
Zhang et al. propose a method based on a greedy discrete search through the space of leaves specifically optimized for fast adversarial example generation \cite{zhang20}.

Another line of work focuses on making tree ensembles more robust. There are multiple approaches: adding generated adversarial examples to the training data (model hardening) \cite{kantchelian2016evasion}, modifying the splitting procedure \cite{hchen19robust,calzavara2020treant,vos21a-groot}, using the framework of optimal decision trees to encode robustness constraints \cite{vos22optimal}, relabeling and pruning the leaves of the trees \cite{vos22relabel}, simplifying the base learner \cite{andriushchenko19} and using a robust 0/1 loss \cite{guo22fast-robust}. Gaining further insights into how evasion attacks target tree ensembles, like those contained in this paper, may inspire novel ways to improve the robustness of learners.

\end{document}